\documentclass[sigconf,balance=false]{acmart}
\usepackage{popets}
\usepackage{amsmath}
\usepackage{algorithmic}
\usepackage[linesnumbered,ruled,vlined]{algorithm2e}
\usepackage{amsthm}
\usepackage[flushleft]{threeparttable}
\usepackage{diagbox}
\SetCommentSty{mycommfont}

\newtheorem{myDef}{Definition}
\newtheorem{myTheo}{Theorem}
\newtheorem{myProp}{Proposition}

\setcopyright{popets}
\copyrightyear{YYYY}

\acmYear{YYYY}
\acmVolume{YYYY}
\acmNumber{X}
\acmDOI{XXXXXXX.XXXXXXX}
\acmISBN{}
\acmConference{Proceedings on Privacy Enhancing Technologies}
\begin{document}

\title[]{VFLGAN: Vertical Federated Learning-based Generative Adversarial Network for Vertically Partitioned Data Publication}


\author{Xun Yuan}
\affiliation{%
  \institution{National University of Singapore}
  \city{Singapore}
  \country{Singapore}}
\email{e0919068@u.nus.edu}

\author{Yang Yang}
\affiliation{%
  \institution{National University of Singapore}
  \city{Singapore}
  \country{Singapore}}
\email{y.yang@u.nus.edu}

\author{Prosanta Gope}
\affiliation{%
  \institution{University of Sheffield}
  \city{Sheffield}
  \country{United Kingdom}
}
\email{p.gope@sheffield.ac.uk}

\author{Aryan Pasikhani}
\affiliation{%
  \institution{University of Sheffield}
  \city{Sheffield}
  \country{United Kingdom}
}
\email{aryan.pasikhani@sheffield.ac.uk}

\author{Biplab Sikdar}
\affiliation{%
  \institution{National University of Singapore}
  \city{Singapore}
  \country{Singapore}}
\email{bsikdar@nus.edu.sg}


\renewcommand{\shortauthors}{Yuan et al.}

\begin{abstract}
  In the current artificial intelligence (AI) era, the scale and quality of the dataset play a crucial role in training a high-quality AI model. However, good data is not a free lunch and is always hard to access due to privacy regulations like the General Data Protection Regulation (GDPR). 
  A potential solution is to release a synthetic dataset with a similar distribution to that of the private dataset. Nevertheless, in some scenarios, it has been found that the attributes needed to train an AI model belong to different parties, and they cannot share the raw data for synthetic data publication due to privacy regulations. 
  In PETS 2023, Xue et al. \cite{Xue01} proposed the first generative adversary network-based model, VertiGAN, for vertically partitioned data publication. However, after thoroughly investigating, we found that VertiGAN is less effective in preserving the correlation among the attributes of different parties. This article proposes a Vertical Federated Learning-based Generative Adversarial Network, VFLGAN, for vertically partitioned data publication to address the above issues. Our experimental results show that compared with VertiGAN, VFLGAN significantly improves the quality of synthetic data. Taking the MNIST dataset as an example, the quality of the synthetic dataset generated by VFLGAN is \textbf{3.2} times better than that generated by VertiGAN w.r.t. the Fréchet Distance.  
  We also designed a more efficient and effective Gaussian mechanism for the proposed VFLGAN to provide the synthetic dataset with a differential privacy guarantee. On the other hand, differential privacy only gives the upper bound of the worst-case privacy guarantee. This article also proposes a practical auditing scheme that applies membership inference attacks to estimate privacy leakage through the synthetic dataset.
\end{abstract}


\maketitle

\section{Introduction}


{In the realm of deep learning (DL), the efficacy of DL models is intimately tied to the quality and scale of the data they are trained on. For instance, the advancements in image perception models can be largely attributed to comprehensive datasets like ImageNet \cite{deng2009imagenet}. Similarly, as demonstrated in \cite{turc2019}, the state-of-the-art language understanding methods thrive on expansive textual datasets like \cite{conneau-etal-2017-supervised}. Furthermore, the success of modern recommendation systems, highlighted in \cite{ma2019learning}, hinges on rich datasets like Netflix ratings \cite{bennett2007netflix}. When these datasets are effectively leveraged with deep learning, the possibilities are boundless, enabling organizations and governments to devise strategies with unprecedented precision and foresight. However, a notable challenge in this realm is the nature of data collection. Often, data is `vertically partitioned', meaning different pieces of customer information are scattered across multiple entities. For example, while a bank may hold a client's financial history, their health records might be with a hospital or insurance firm. As per studies like \cite{8667703, Xue01}, integrating such dispersed attributes can provide a holistic view of customers, thus significantly enhancing decision-making processes.}

{Despite the evident benefits of integrating dispersed data attributes, practical implementation is often hampered due to privacy concerns. Moreover, stringent data protection regulations like GDPR \cite{voigt2017eu} further curb the sharing of customer data between entities. One potential avenue to navigate these challenges is the publication of synthetic data that mirror the distribution of private data without disclosing any actual private information. However, this solution is not without its vulnerabilities. Adversaries have devised methods that leverage synthetic datasets to glean insights into the corresponding private datasets. Cases in point are the Membership Inference (MI) attacks and attribute inference techniques proposed in \cite{Theresa01}. To counteract such vulnerabilities, Differential Privacy (DP) \cite{dwork2006} offers a promising strategy for privacy protection. By infusing DP principles into synthetic data publication, one can provide these synthetic datasets with a robust privacy assurance, fortifying them against threats like MI attacks.}

{Building upon the foundational discussions on data privacy and the challenges of vertically partitioned data publication, this paper addresses the limitations of existing DP methods in managing vertically partitioned data. Notably, conventional DP solutions like DistDiffGen \cite{6517175}, and DPLT \cite{8667703} are tailored for specific kinds of datasets. Generative Adversarial Networks (GANs) \cite{goodfellow2014generative}, recognized for their ability to replicate original data distributions, provide a novel approach for universal kinds of datasets. Besides, Federated Learning (FL) can help to comply with GDPR's data localization requirements.}
{Through the integration of GANs with Vertical Federated Learning (VFL) \cite{liu2022vertical}, our proposed VFLGAN model surpasses VertiGAN \cite{Xue01} in achieving enhanced attribute correlation. Furthermore, we extend this innovation with DP-VFLGAN, incorporating an optimized Gaussian mechanism tailored to the VFL context, thereby diverging from the DP mechanisms applied in \cite{xie2018differentially, zhang2018differentially, jordon2018pate, chen2020gs}, and advancing the field of privacy-preserving data publication.}

{While DP can offer worst-case privacy assurances \cite{mehner2021towards}, most real-world datasets do not contain the worst-case data record. Importantly, the prime concern for data owners and regulatory bodies is gauging the actual information leakage, as determined by real-world privacy attack simulations. Several privacy metrics \cite{yale2019assessing, lu2019empirical, giomi2023unified} and attack strategies \cite{hayes2017logan, hilprecht2019monte, chen2020gan, van2023membership} aimed at synthetic datasets currently exist. Nevertheless, these metrics don't align with DP principles, and many of the attack strategies make assumptions about reference or auxiliary data.
In light of these issues, this paper introduces an innovative auditing scheme. The proposed scheme aligns with DP principles and draws inspiration from privacy games \cite{Jiayuan01} and shadow models \cite{Reza01}, allowing for a robust assessment of information leakage for any given data record.
Moreover, we experimentally show that the proposed auditing scheme is more robust than current attacks \cite{hayes2017logan, hilprecht2019monte, chen2020gan, van2023membership} that target synthetic datasets. 
}
{Readers can refer to Appendix \ref{sec:literature review} for a detailed literature review.
}

{
In summary, current literature offers several methodologies to address vertically partitioned data publication, but they come with notable shortcomings.  
In this paper, we propose VFLGAN as an effective solution to mitigate those shortcomings. 
Additionally, recognizing the privacy risks inherent in synthetic datasets, we incorporate a differentially private mechanism to ensure that VFLGAN satisfies a DP guarantee. While several mechanisms exist to provide centralized GANs with DP guarantees, as discussed in Appendix \ref{sec:literature review DP}, these mechanisms are unsuitable for VFLGAN due to the shared discriminator. To address this, we design a variant of the Gaussian mechanism tailored for DP-VFLGAN. 
Last, it is important to note that DP provides a privacy guarantee for the worst case, which rarely (if not never) exists in real-world datasets. Therefore, a practical auditing scheme is required to accurately estimate the privacy risk of training data. Existing privacy leakage measurements and attacks, which do not adhere to DP principles or rest on unrealistic assumptions, are inadequate for this purpose. In this paper, we design a novel and practical auditing scheme to effectively estimate the privacy risk of any given record.
}

\subsection{Contributions} \label{sec:contributions}

This paper outlines the following contributions to address the aforementioned research gaps.  
\begin{itemize}
    \item[$\bullet$] Through comprehensive investigation, we identified a critical limitation in a recently proposed GAN-based method for vertically partitioned data publication (presented at PETS 2023 \cite{Xue01}). Specifically, this model fails to effectively learn the correlation among attributes across different parties.
    \item[$\bullet$] We introduce the \emph{first} Vertical Federated Learning (VFL)-based Generative Adversarial Network, named VFLGAN. This novel model is adept at learning correlations among attributes between different parties (as demonstrated in Fig. \ref{fig:demo}) and is equipped to handle both continuous and categorical attributes efficiently.
    \item[$\bullet$] A new Gaussian mechanism has been developed, equipping DP-VFLGAN with a $(\epsilon,\delta)$-Differential Privacy (DP) guarantee. This enhancement ensures heightened privacy protection in data publication.
    \item[$\bullet$] We propose a pragmatic auditing scheme - a privacy leakage measurement - that operates without reliance on unrealistic assumptions. This scheme quantitatively assesses the privacy risk of synthetic datasets.
    \item[$\bullet$] Extensive experiments were conducted to evaluate the quality of synthetic datasets generated by VFLGAN rigorously. We applied multiple metrics for a thorough assessment. Furthermore, the effectiveness of the proposed Gaussian mechanism was also evaluated using our innovative auditing method, demonstrating its efficacy in practical scenarios.

    \item[$\bullet$] The code$\footnote{\url{https://github.com/YuanXun2024/VFLGAN}}$ will be released at publication.
    
\end{itemize}
Readers can refer to \textbf{Appendix} \ref{sec:summary of notation} for a summary of notation.

\begin{figure}[htbp]
\center{\includegraphics[width=0.8\linewidth, scale=1.]{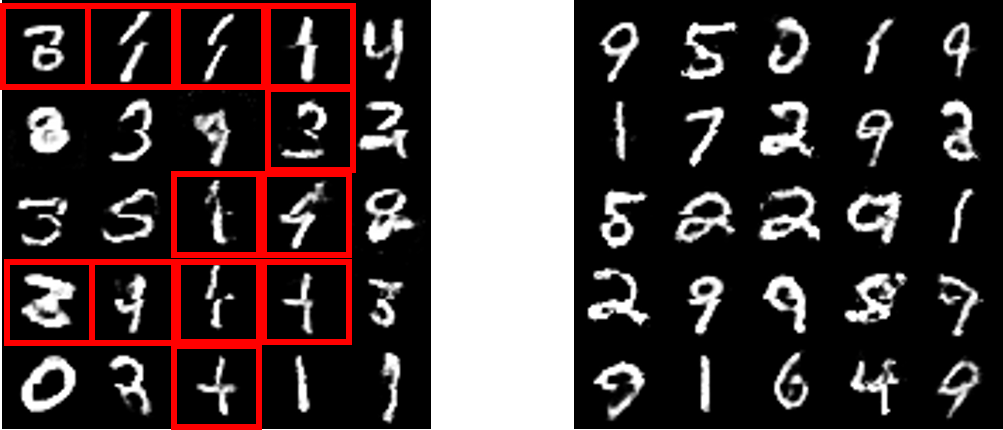}}
\caption{This figure shows synthetic samples generated by GANs trained on vertically partitioned MNIST data (the digits are split evenly into upper and lower halves). The left figure displays samples generated by VertiGAN \cite{Xue01}, highlighting unrecognizable and discontinuous digits. The right figure shows samples generated by the proposed VFLGAN.}
\label{fig:demo}
\vspace{-6mm} 
\end{figure}

\section{Prelimineries}
This section provides preliminaries of DP and GANs to facilitate a comprehensive understanding of the proposed VFLGAN.
Besides, the training process of VertiGAN is introduced briefly.

\subsection{Differential Privacy}
Differential privacy \cite{dwork2006} provides a rigorous privacy guarantee that can be quantitatively analyzed. The pure $\epsilon$-DP is defined as follows.
\begin{myDef}
\label{Def1}
    ($\epsilon$-DP). A randomized mechanism $f: D \rightarrow R$ satisfies $\epsilon$-differential privacy ($\epsilon$-DP) if for any adjacent $D, D^\prime \in \mathcal{D}$ and $S \subset R$ 
    
    \centerline{$Pr[f(D)\in S] \leq e^\epsilon Pr[f(D^\prime)\in S]$.}
\end{myDef}
In the literature, the most commonly used DP is a relaxed version of the pure DP, which allows the mechanism to satisfy $\epsilon$-DP most of the time but not satisfy $\epsilon$-DP with a small probability, $\delta$. The relaxed version, $(\epsilon,\delta)$-DP \cite{dwork2006our}, is defined as follows.
\begin{myDef}
\label{Def2}
    (($\epsilon,\delta$)-DP). A randomized mechanism $f: D \rightarrow R$ provides ($\epsilon,\delta$)-differential privacy (($\epsilon,\delta$)-DP) if for any adjacent $D, D^\prime \in \mathcal{D}$ and $S \subset R$ 
    
    \centerline{$Pr[f(D)\in S] \leq e^\epsilon Pr[f(D^\prime)\in S]+\delta$.}
\end{myDef}

In \cite{mironov2017renyi}, the $\alpha$-R\'enyi divergences between $f(D)$ and $f(D^\prime)$ are applied to define R\'enyi Differential Privacy (RDP) which is a generalization of differential privacy. ($\alpha,\epsilon(\alpha)$)-RDP is defined as follows.
\begin{myDef}
\label{Def3}
    ($(\alpha,\epsilon(\alpha))$-RDP). A randomized mechanism $f:D \rightarrow R$ is said to have $\epsilon(\alpha)$-R\'enyi differential privacy of order $\alpha$, or $(\alpha,\epsilon(\alpha))$-RDP for short if for any adjacent $D, D^\prime \in \mathcal{D}$ it holds that
    
    \centerline{$
D_\alpha\left(f(D) \| f\left(D^{\prime}\right)\right)=\frac{1}{\alpha-1} \log \mathbb{E}_{x \sim f(D)}\left[\left(\frac{\operatorname{Pr}[f(D)=x]}{\operatorname{Pr}\left[f\left(D^{\prime}\right)=x\right]}\right)^{\alpha-1}\right] \leq \epsilon
$.}
\end{myDef}

\cite{mironov2017renyi} proves that the Gaussian mechanism can guarantee an RDP.
\begin{myProp}
\label{Prop: Gaussian mechenism}
    (Gaussian Mechanism) Let $f: D \rightarrow R$ be an arbitrary function with sensitivity being     
    
    \centerline{$\Delta_2 f=\max _{D, D^{\prime}}\left\|f(D)-f\left(D^{\prime}\right)\right\|_2$}     
    \noindent for any adjacent $D, D^\prime \in \mathcal{D}$. The Gaussian Mechanism $M_\sigma$,     
    
    \centerline{$\mathcal{M}_\sigma(\boldsymbol{x})=f(\boldsymbol x)+\mathcal{N}\left(0, \sigma^2 I\right)$}     
    \noindent provides $\left(\alpha, \alpha \Delta_2 f^2 /2\sigma^2\right)$-RDP.
\end{myProp}

The (R)DP budget should be accumulated if we apply multiple mechanisms to process the data sequentially as we train deep learning (DL) models for multiple iterations. We can calculate the accumulated RDP budget by the following proposition \cite{mironov2017renyi}.

\begin{myProp}
\label{Prop: RDP composition}
    (Composition of RDP) Let $f : D \rightarrow R_1$ be $(\alpha, \epsilon_1)$-RDP and $g : R_1 \times D \rightarrow R_2$ be $(\alpha,\epsilon_2)$-RDP, then the mechanism defined as $(X, Y)$, where $X \sim f(D)$ and $Y \sim g(X, D)$, satisfies $(\alpha, \epsilon_1 + \epsilon_2)$-RDP.
\end{myProp}
According to the following proposition, RDP can be converted to $(\epsilon,\delta)$-DP and the proof can be found in \cite{mironov2017renyi}.
\begin{myProp}
\label{Prop: RDP to DP}
    (From RDP to $(\epsilon,\delta)$-DP) If $f$ is an $(\alpha, \epsilon(\alpha))$-RDP mechanism, it also satisfies $(\epsilon(\alpha) + \frac{log1/\delta}{\alpha-1}, \delta)$-DP for any $0 < \delta < 1$.
\end{myProp}
According to Proposition \ref{Prop: RDP to DP}, given a $\delta$ we can get a tight $(\epsilon^\prime, \delta)$-DP bound by
\begin{equation}
    \label{eq:RDP to DP}
    \epsilon^\prime = \min_\alpha(\epsilon(\alpha) + \frac{log1/\delta}{\alpha-1}).
\end{equation}
\cite{wang2019subsampled} provides a tight upper bound on RDP by considering the combination of the subsampling procedure and random mechanism. This is important for differentially private DL since DL models are mostly updated according to a subsampled mini-batch of data. The enhanced RDP bound can be calculated according to the following proposition, and the proof can be found in \cite{wang2019subsampled}.

\begin{myProp}
\label{theorem: subsample RDP}
(RDP for Subsampled Mechanisms). Given a dataset of $n$ points drawn from a domain $\mathcal{X}$ and a (randomized) mechanism $\mathcal{M}$ that takes an input from $\mathcal{X}^m$ for $m \leq n$, let the randomized algorithm  $\mathcal{M} \circ$ \text{subsample} be defined as (1) subsample: subsample without replacement $m$ datapoints of the dataset (sampling rate $\gamma$ = $m/n$), and (2) apply $\mathcal{M}$: a randomized algorithm taking the subsampled dataset as the input. For all integers $\alpha$ $\geq$ 2, if $\mathcal{M}$ obeys ($\alpha, \epsilon(\alpha)$)-RDP, then this new randomized algorithm $\mathcal{M} \circ$ \text{subsample} obeys ($\alpha, \epsilon^\prime(\alpha)$)-RDP where,

\centerline{$
\begin{aligned}
\epsilon^{\prime}(\alpha) \leq & \frac{1}{\alpha-1} \log \left(1+\gamma^2\left(\begin{array}{c}
\alpha \\
2
\end{array}\right) \min \right. \\
&\left\{4\left(e^{\epsilon(2)}-1\right),  e^{\epsilon(2)} \min \left\{2,\left(e^{\epsilon(\infty)}-1\right)^2\right\}\right\} \\
+ & \left.\sum_{j=3}^\alpha \gamma^j\left(\begin{array}{c}
\alpha \\
j
\end{array}\right) e^{(j-1) \epsilon(j)} \min \left\{2,\left(e^{\epsilon(\infty)}-1\right)^j\right\}\right)
\end{aligned}
$}
\end{myProp}

Last, the following post-processing theorem \cite{dwork2014algorithmic} is convenient, with which we can say a framework satisfies DP or RDP if any intermediate function of the framework satisfies DP or RDP.
\begin{myProp}
\label{theorem: post-processing}
(Post-processing). If $f(\cdot)$ satisfies ($\epsilon,\delta$)-DP, $g(f(\cdot))$ will satisfy ($\epsilon,\delta$)-DP for any function $g(\cdot)$. Similarly, if $f(\cdot)$ satisfies ($\alpha,\epsilon$)-RDP, $g(f(\cdot))$ will satisfy ($\alpha,\epsilon$)-RDP for any function $g(\cdot)$.
\end{myProp}

\subsection{Generative Adversarial Models} \label{sec:preliminaries of GAN}

Given a dataset $X$ where the data record $x \in X$ follows the distribution $P$, the generator of GAN, $G$, aims to generate synthetic data $\tilde x$, $\tilde x=G(z)$, that follows the distribution $P_{G(z)}$ similar to $P$. The input $z$ is sampled from a simple distribution, such as the uniform distribution or a Gaussian distribution. The above object can be achieved with the help of a discriminator, $D$. The generator and discriminator are trained through a competing game, where the discriminator is trained to distinguish $x$ and $\tilde x$, and the generator is trained to generate high-quality $\tilde x$ to fool the discriminator. 
The game between the generator and the discriminator can be formally expressed as the following min-max objective \cite{goodfellow2014generative},
\begin{equation}
\label{eq:min-max objective}
\min _G \max _D \underset{\boldsymbol{x} \sim P}{\mathbb{E}}[\log (D(\boldsymbol{x}))]+\underset{\tilde{\boldsymbol{x}} \sim P_{G(\boldsymbol{z})}}{\mathbb{E}}[\log (1-D(\tilde{\boldsymbol{x}}))].
\end{equation}
As proved in \cite{goodfellow2014generative}, this objective leads to minimizing the Jensen- Shannon divergence between $P$ and $P_{G(\boldsymbol{z})}$. However, the training process of optimizing the objective \eqref{eq:min-max objective} is unstable due to discriminator saturating, which results in vanishing gradients. \cite{arjovsky2017wasserstein} pointed out that the Jensen–Shannon divergence is not continuous and does not provide usable gradients and proposed a new objective, i.e., minimizing the Wasserstein-1 distance between $P$ and $P_{G(\boldsymbol{z})}$, which is continuous everywhere and differentiable almost everywhere under mild assumptions. Based on the new objective, Wasserstein GAN (WGAN) is proposed in \cite{arjovsky2017wasserstein}. Following this idea, subsequent works \cite{gulrajani2017improved, wei2018improving, wu2018wasserstein} propose variants of WGAN to improve the quality of generated data. Same as the previous works \cite{Xue01,chen2020gs}, we adopt the optimization objectives of WGAN\_GP \cite{gulrajani2017improved} for the proposed VFLGAN as,
\begin{equation}
\label{eq:obj_wgan_gp}
    \begin{aligned}
        &\min _D -\mathbb{E}[D(\boldsymbol{x})]+\mathbb{E}[D(\tilde{\boldsymbol{x}})] 
         + \lambda \mathbb{E}\left[\left(\|\nabla D(\hat{\boldsymbol{x}})\|_2-1\right)^2\right], \\
         &\max_G \mathbb{E}[D(\tilde{\boldsymbol{x}})],
    \end{aligned} 
\end{equation}
where $\hat{\boldsymbol{x}}=\beta \boldsymbol{x} + (1-\beta)\tilde{\boldsymbol{x}}$ and $\beta \sim \mathbb{U}(0,1)$.

\subsection{VertiGAN} \label{sec:vertigan}

{Figure \ref{fig:verigan} shows the training framework of VertiGAN for the two-party scenario. Each party shares the same generator backbone ($G_b$). Party $i$ maintains a private generator head ($G_{hi}$) and a private discriminator ($D_i$). The local update in each party is the same process as WGAN\_GP \cite{gulrajani2017improved}, except for the update of the generator backbone ($G_b$). Thus, same as \cite{gulrajani2017improved}, the optimization objective of party $i\in\{1,2\}$ can be expressed as,
\begin{equation}
\label{eq:obj_wgan_gp}
    \begin{aligned}
        &\min _{D_i} -\mathbb{E}[D_i(\boldsymbol{x_i})]+\mathbb{E}[D_i(\tilde{\boldsymbol{x}}_i)] 
         + \lambda \mathbb{E}\left[\left(\|\nabla D_i(\hat{\boldsymbol{x}}_i)\|_2-1\right)^2\right], \\
         &\max_{G_i} \mathbb{E}[D_i(\tilde{\boldsymbol{x}}_i)],
    \end{aligned} 
\end{equation}
where $\hat{\boldsymbol{x}}_i=\beta \boldsymbol{x}_i + (1-\beta)\tilde{\boldsymbol{x}}_i$ and $\beta \sim \mathbb{U}(0,1)$. For the update of the generator backbone ($G_b$), each party sends the local gradients of $G_b$, i.e., $\mathcal{G}^{i}_{G_b}$, to the server. The server summarizes the local gradients as,
\begin{equation}
\mathcal{G}_{G_b} = \mathcal{G}^{1}_{G_b} + \mathcal{G}^{2}_{G_b},
\label{eq:hfl}
\end{equation}
and sends $\mathcal{G}_{G_b}$ to the local parties. Then, the local party updates the local $G_b$ with $\mathcal{G}_{G_b}$. Equation \eqref{eq:hfl} is the main idea of horizontal federated learning (HFL).
}

\begin{figure}[htbp]
\center{\includegraphics[width=\linewidth, scale=1.]{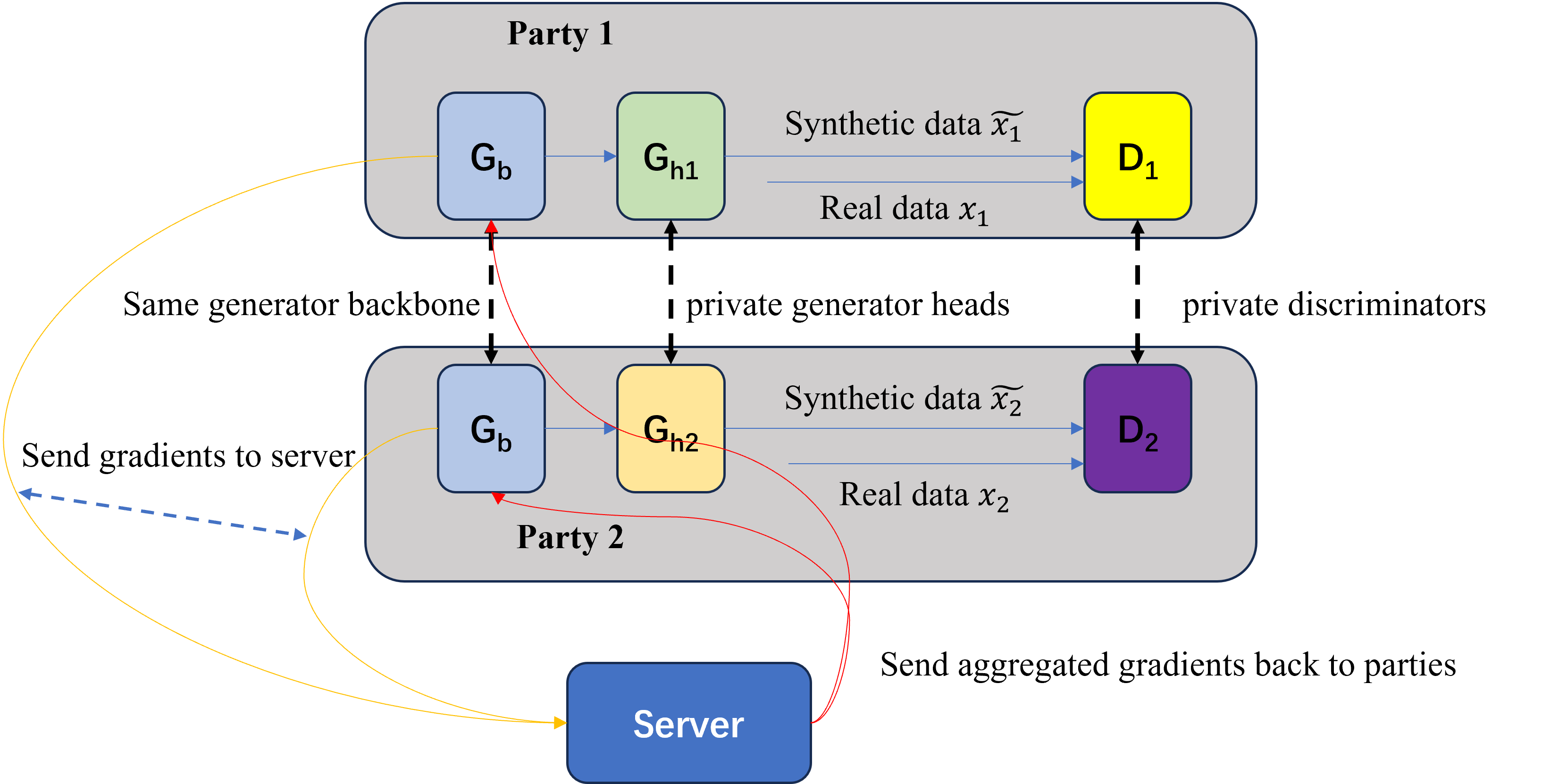}}
\caption{This figure shows the framework of VertiGAN.}
\label{fig:verigan}
\end{figure}

\section{Proposed Vertical Federated Learning-based GAN}

{
This section first introduces our system model. According to the system model, we formulate the vertically partitioned data publication problem as a min-max optimization problem to train the proposed VFLGAN. Subsequently, we specify the architecture of the proposed VFLGAN for the two-party case. Following the architectural overview, the training process of VFLGAN is described in detail. Then, we introduce the differentially private version of VFLGAN, i.e., DP-VFLGAN. The section concludes by delineating the differences between VertiGAN and VFLGAN.
}

\subsection{System Model and Problem Formulation}

Our system model considers a similar scenario as discussed in \cite{Xue01}, where there are $M$ non-colluding parties. Each party $P_i$ maintains a private dataset $X_i \in \mathbb{R}^{N \times |A_i|}$ with $N$ records where $A_i$ denotes the attribute set of $X_i$. 
Now, we consider the following assumptions for our proposed system model.

\textbf{Assumption 1} Records in different datasets $X_i$ with the same index belong to the same object, which can be achieved by applying private set intersection protocols \cite{huang2012private,chen2017fast} in practice. To facilitate the alignment of training inputs across parties without direct data sharing, a pseudorandom number generator can be employed during the training process.

\textbf{Assumption 2} There is no common attribute among the parties. 

\textbf{Assumption 3} The local parties and the central server are honest but curious, i.e., correctly follow the protocols but try to infer sensitive information from other
parties.

With the above assumptions, $M$ private datasets can be combined to construct a new dataset $X$, i.e., $X=[X_1,X_2,\cdots,X_M] \in \mathbb{R}^{N \times \sum_{i=1}^M|A_i|}$, where $[\cdots]$ denotes a concatenation function. The parties aim to generate a synthetic dataset $\tilde X$ in which each record $\tilde{\boldsymbol{x}} \in \tilde{{X}}$ follows a similar distribution to that of $\boldsymbol{x} \in X$,
\begin{equation}
    \label{eq:formulation}
    P_{{\tilde{\boldsymbol{x}}}}\approx P_{\boldsymbol{x}},
\end{equation}
while keeping the local data secret from other parties.

Here, we use GAN-based generators to generate synthetic records that satisfy \eqref{eq:formulation}. 
Figure \ref{fig:system model} illustrates our system model in detail. There is a private dataset $X_i$, a local generator $G_i$, and a local discriminator $D_i$ in each party $P_i$. The server maintains a shared discriminator $D_s$ and is responsible for combining the synthetic data from all the parties. Solid lines between the server and parties represent the communication during the generalisation of the synthetic dataset (inference period), and the dashed lines represent the communication during the training process. In our system model, the private dataset $X_i$ can only be accessed by the corresponding discriminator $D_i$. The shared discriminator $D_s$ aims to guide the local generators to learn the correlation among attributes of different parties. During the training process, each $D_i$ sends its intermediate feature to $D_s$ and $D_s$ return gradients to update $D_i$ and $G_i$. During the inference process, the local generator $G_i$ generates partial synthetic records, ${\tilde{\boldsymbol{x}}}_i=G_i(\boldsymbol{z})$, where $\boldsymbol{z}$ is the same for all parties achieved by a pseudorandom number generator at each local party. Then, the partial synthetic records are concatenated in the server to get a complete synthetic record, $\tilde x = [{\tilde{\boldsymbol{x}}}_1,{\tilde{\boldsymbol{x}}}_2,\cdots,{\tilde{\boldsymbol{x}}}_M]$. According to our system model, the objective \eqref{eq:formulation} can be expressed as the following,
\begin{equation}
    \label{eq:formulation system model}
    P_{{\tilde{\boldsymbol{x}}}}\approx P_{\boldsymbol{x}}, \quad \tilde{\boldsymbol{x}} = [G_1(\boldsymbol{z}),G_2(\boldsymbol{z}),\cdots,G_M(\boldsymbol{z})].
\end{equation}
Thus, the problem targeted by this paper is to train $M$ generators that satisfy \eqref{eq:formulation system model}. As mentioned in Section \ref{sec:preliminaries of GAN}, this problem can be transformed into a min-max optimization problem to obtain such generators. According to the optimization objectives of WGAN\_GP \eqref{eq:obj_wgan_gp}, the vertically partitioned data publication problem of our system model can be formulated as,
\begin{align}
\label{eq:obj_system_model}
        &\min _{D_1,\cdots,D_M,D_s} \sum_{i=1}^M\mathcal{L}(D_i,{\tilde{\boldsymbol{x}}}_i,\boldsymbol{x}_i)+\lambda_1\mathcal{L}(D_s,{\tilde{\boldsymbol{f}}},\boldsymbol{f}), \\
\label{eq:obj_system_model_2}
         &\max_{G_1,\cdots,G_M} \sum_{i=1}^M\mathbb{E}[D_i({\tilde{\boldsymbol{x}}}_i)] + \lambda_2\mathbb{E}[D_s({\tilde{\boldsymbol{f}}})],
\end{align}
where $\mathcal{L}(D,\tilde{\boldsymbol{x}},\boldsymbol{x})\triangleq -\mathbb{E}[D(\boldsymbol{x})]+\mathbb{E}[D(\tilde{\boldsymbol{x}})] 
         + \lambda \mathbb{E}\left[\left(\|\nabla D(\hat{\boldsymbol{x}})\|_2-1\right)^2\right]$, ${\tilde{\boldsymbol{f}}}$ and $\boldsymbol{f}$ are the concatenations of the intermediate features of $D_1$ to $D_M$ when the inputs are synthetic data and real data respectively, and $\lambda_1$ and $\lambda_2$ are balancing coefficients.

\begin{figure}[htbp]
\center{\includegraphics[width=\linewidth, scale=1.2]{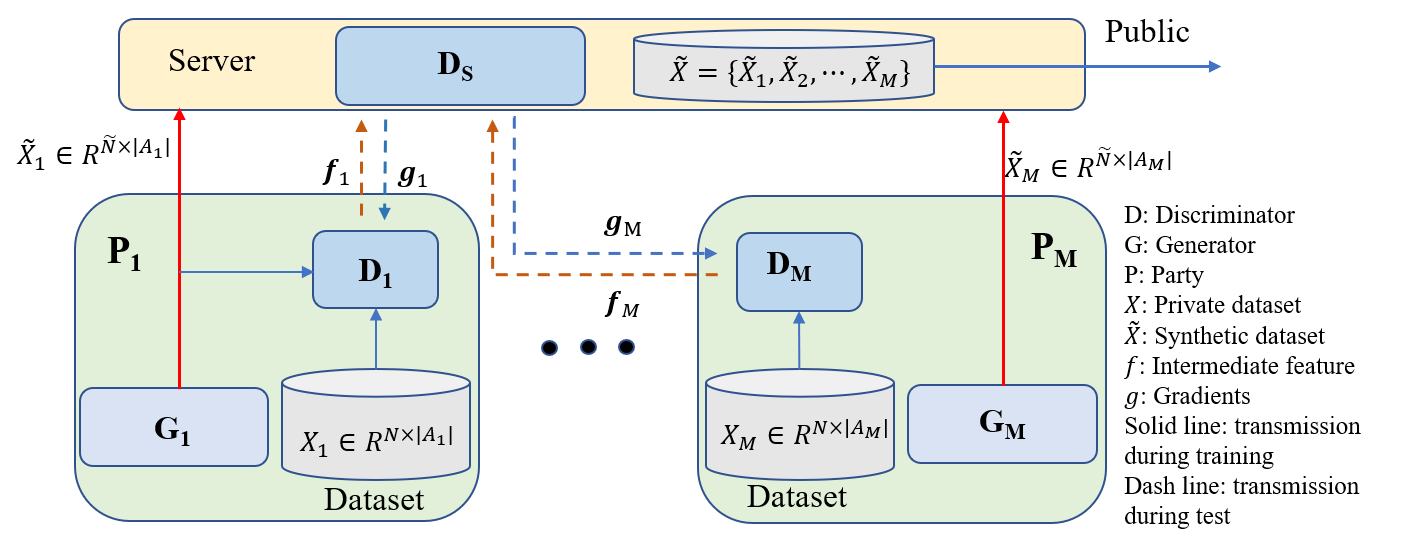}}
\caption{System Model.}
\label{fig:system model}
\end{figure}



\subsection{Overview of VFLGAN for Two-party Case}
Figure \ref{fig:vflmodel} illustrates the proposed VFLGAN for the two-party scenario. There are two private generators, $G_1$ and $G_2$, two private discriminators, $D_1$ and $D_2$, and one shared discriminator, $D_s$, in the framework. The generators produce synthetic data with a vector of Gaussian noise $\boldsymbol{z}$, i.e., ${\tilde{\boldsymbol{x}}}_i = G_i(\boldsymbol{z})$. The discriminators, $D_i$ where $i\in \{1,2\}$, are trained to distinguish synthetic data ${\tilde{\boldsymbol{x}}}_i$ and real data $\boldsymbol{x}_i$. Thus, the gradients of $D_i$ on ${\tilde{\boldsymbol{x}}}_i$ can guide $G_i$ to generate synthetic data with a similar distribution to the real data $\boldsymbol{x}_i$, i.e.,
\begin{equation}
\label{eq: 5}
    P_{G_i(\boldsymbol{z})} \approx P_{\boldsymbol{x}_i}. 
\end{equation}

{
Figure \ref{fig:VFLGAN_architecture} in the Appendix shows the detailed structure of the VFLGAN discriminators. In $D_i$ $i\in\{1,2\}$, the layers before the intermediate feature construct the first part of $D_i$ in Fig. \ref{fig:vflmodel}, and the layers after the intermediate feature construct the second part of $D_i$. The intermediate feature ($\boldsymbol{f}_i$) is the output of the first part of $D_i$, i.e.,
\begin{equation}
    \boldsymbol{f}_i = D_i^1(\boldsymbol{x}_i) \quad i\in \{1,2\},
\end{equation}
where $D_i^1$ denotes the first part of $D_i$ and $\boldsymbol{x}_i$ denotes the input of $D_i$.
The intermediate features are transmitted to the server where the concatenation of the intermediate features is input to $D_s$. $D_s$ is trained to distinguish $[{\tilde{\boldsymbol{x}}}_1, {\tilde{\boldsymbol{x}}}_2]$ and $[\boldsymbol{x}_1,\boldsymbol{x}_2]$ by optimizing $\mathcal{L}(D_s,{\tilde{\boldsymbol{f}}},\boldsymbol{f})$ in \eqref{eq:obj_system_model}. As a result, by optimizing \eqref{eq:obj_system_model_2}, $D_s$ can guide $G_1$ and $G_2$ to learn the correlation between $\boldsymbol{x}_1$ and $\boldsymbol{x}_2$ and generate better synthetic data, i.e., 
\begin{equation}
\label{eq: 6}
    P_{[G_1(\boldsymbol{z}),G_2(\boldsymbol{z})]} \approx P_{[\boldsymbol{x}_1,\boldsymbol{x}_2]}. 
\end{equation}
}

\begin{figure}[htbp]
\center{\includegraphics[width=\linewidth, scale=1.2]{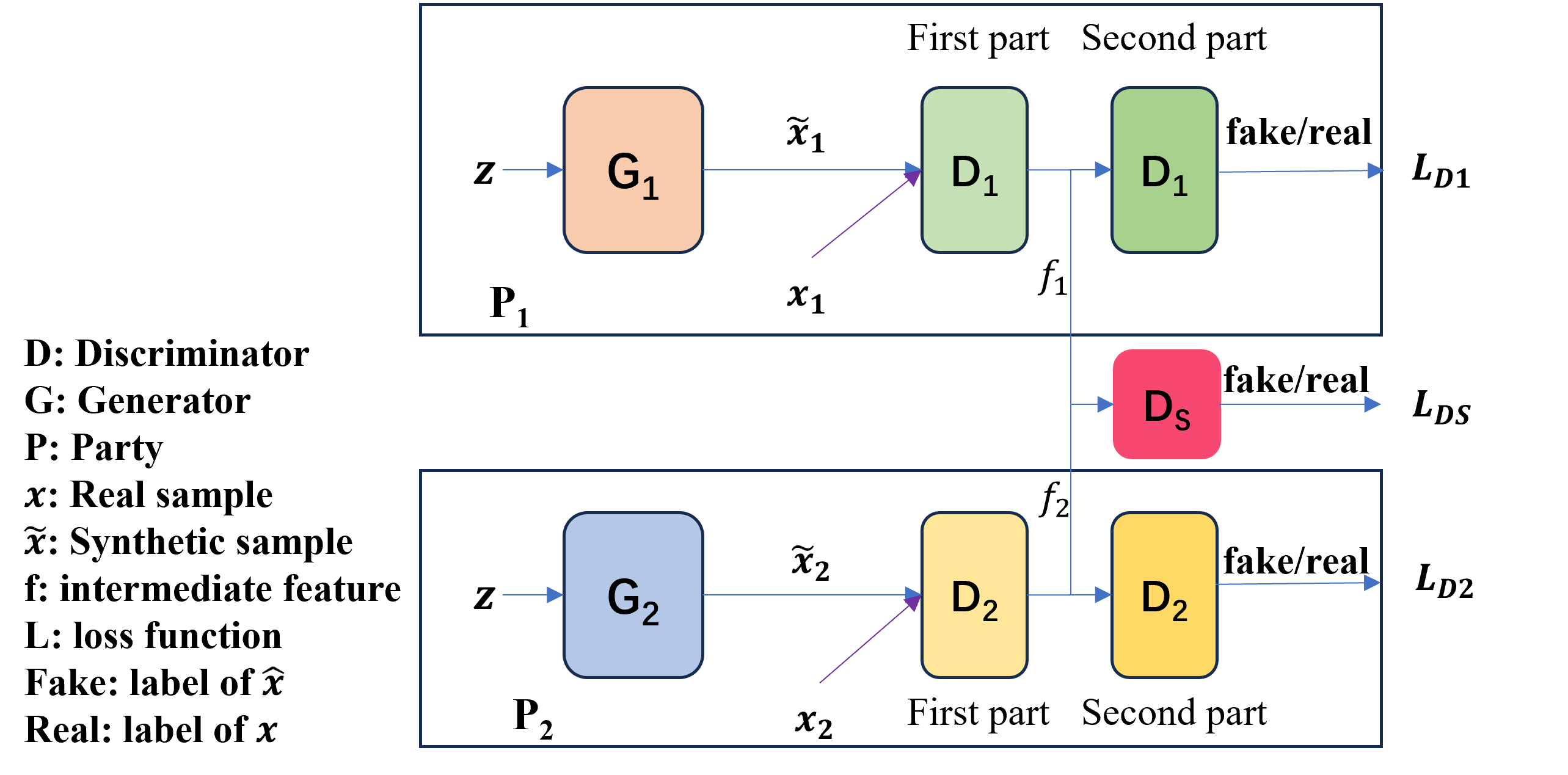}}
\caption{Framework of the proposed VFLGAN.}
\label{fig:vflmodel}
\end{figure}

{
Concerns may arise regarding the potential for privacy leakage through intermediate features within a neural network. However, reconstructing the input from the neural network's output, without access to the model's parameters, presents a difficult ill-posed problem \cite{dosovitskiy2016generating}. For instance, the authors in \cite{Yang-ccs-2019} attempt to invert neural networks used for classification, where the adversary is assumed to possess an auxiliary dataset and the ability to interact with the model by submitting requests and receiving responses. Nevertheless, in our scenario, such access to the model or the ability to submit requests is precluded. Another study \cite{9833677} endeavours to reconstruct training data from the outputs of a trained neural network, with the adversary having access to all but one record in the training dataset and white-box access to the model. Despite the strong assumptions made in \cite{Yang-ccs-2019,9833677}, the reconstructed images are notably blurred, rendering them less precise than the original images. While a blurred image might still reveal some private information, a blurred tabular record is likely to disclose minimal privacy details. Furthermore, to address these concerns, we have developed an auditing scheme, detailed in Section \ref{sec:Auditing Scheme for Intermediate Features}, designed to assess the extent of privacy leakage through intermediate features.
}

\subsection{Training Process of VFLGAN}

Following the training procedure of previous works \cite{goodfellow2014generative, gulrajani2017improved}, we optimize the discriminators and generators in sequence, i.e., we first optimize the discriminators for $T_d$ iterations and then optimize the generators for one iteration. According to our optimization objective \eqref{eq:obj_system_model}, the loss function of $D_s$ is $\mathcal{L}(D_s,{\tilde{\boldsymbol{f}}},\boldsymbol{f})$ and the loss function for $D_1$ and $D_2$ is $\mathcal{L}(D_i,{\tilde{\boldsymbol{x}}}_i,\boldsymbol{x}_i)$. In the remaining paper, we use $\mathcal{L}_{D_s}$ and $\mathcal{L}_{D_i}$ to denote the above loss functions for abbreviation. The gradients of $D_s$ parameters can be calculated according to its loss function by,
\begin{equation}
\label{eq:gradient_ds}
    \mathcal{G}_{D_s} = \nabla_{\boldsymbol{\theta}_{D_s}}\mathcal{L}_{D_s}.
\end{equation}
Note that $\mathcal{L}_{D_s}$ also contributes to the gradients of the first part of $D_1$ and $D_2$ parameters. Thus, the gradients of $D_1$ and $D_2$ parameter can be calculated by,
\begin{align}
\label{eq:gradient_d1}
    \mathcal{G}_{D_1} = \nabla_{\boldsymbol{\theta}_{D_1}}\mathcal{L}_{D_1}+\nabla_{\boldsymbol{\theta}_{D_1}}\mathcal{L}_{D_s}, \\
\label{eq:gradient_d2}
    \mathcal{G}_{D_2} = \nabla_{\boldsymbol{\theta}_{D_2}}\mathcal{L}_{D_2}+\nabla_{\boldsymbol{\theta}_{D_2}}\mathcal{L}_{D_s}.
\end{align}
Finally, the parameters of the discriminators can be updated by,
\begin{equation}
    \label{eq:udgrade_d}
    \boldsymbol{\theta}_{D_i} = \boldsymbol{\theta}_{D_i} - \eta_{D_i} \mathcal{G}_{D_i} \quad i\in \{1,2,S\},
\end{equation}
where $\eta_{D_i}$ denotes the learning rate of $D_i$.

Notably, when we optimize the generators according to the objective function \eqref{eq:obj_system_model_2}, the output of the discriminators is maximised.
However, the optimizers in the DL field, like Adam \cite{kingma2014adam}, usually minimize the loss function. Thus, the loss function for the generators is derived as,
\begin{equation}
    \label{eq:g_loss}
    \mathcal{L}_{G} = -\sum_{i=1}^2\mathbb{E}[D_i(G_i(\boldsymbol{z})]-\mathbb{E}[D_s([\boldsymbol{f}_1,\boldsymbol{f}_2])],
\end{equation}
where $\boldsymbol{f}_i$ denotes the intermediate feature of $D_i$ as shown in Fig. \ref{fig:vflmodel}. Then, the gradients and parameter update of $G_i$ can be calculated as follows, 
\begin{align}
\label{eq:gradient_g}
    &\mathcal{G}_{G_i} = \nabla_{\boldsymbol{\theta}_{G_i}}\mathcal{L}_{G}, \\
\label{eq:udgrade_g}
    &\boldsymbol{\theta}_{G_i} = \boldsymbol{\theta}_{G_i} - \eta_{G_i} \mathcal{G}_{G_i} \quad i\in\{1,2\}.
\end{align}
\emph{Note} $D_s$ also contributes to the loss function of $G_1$ and $G_2$, through which $G_1$ and $G_2$ learn to generate synthetic data satisfying \eqref{eq: 6}.

The training process of the VFLGAN is summarized in Algorithm \ref{alg:Training DP-VFL-GAN}. First, we train the discriminators for $T_d$ iterations. During each iteration, we subsample a mini-batch of real data, $[\boldsymbol{x}_1^B,\boldsymbol{x}_2^B]$, where $B$ denotes the batch size, from the training dataset and generate a mini-batch of synthetic data $[\tilde{\boldsymbol{x}}_1^B,\tilde{\boldsymbol{x}}_2^B]$ with $G_1$ and $G_2$. And the discriminators are trained to distinguish real data and synthetic data. Second, we train the generators to generate more realistic synthetic data for one iteration. The first and second steps repeat for $T_{max}$ epochs, and the algorithm outputs trained $G_1$ and $G_2$. 

\begin{algorithm}[htbp]
\caption{Training Process of (DP-)VFLGAN}\label{alg:Training DP-VFL-GAN}
\footnotesize
\begin{algorithmic}[1]
\REQUIRE $G_1$ and $G_2$: generators; $D_1$ and $D_2$: discriminators; $\boldsymbol{f}_i$: the intermediate feature of $D_i$; $n_{D_1}$ and $n_{D_2}$: number of layers of $D_1$ and $D_2$; $D_s$: shared discriminator; $X$: dataset $B$: batch size; $T_{max}$: maximum training epochs; $T_d$: discriminators' update steps; $\eta$: learning rate; $l$: latent dimension; $\theta$: parameters of VFLGAN; $\mathcal{G}_{D_i}^1$: gradients of parameters of the first layer of $D_i$.
\ENSURE Trained $G_1$ and $G_2$.

\noindent{\rule{0.95\linewidth}{0.4pt}}

 Initialize the generators and discriminators; \\
\For{$epoch$ in $\{1,2,\cdots,T_{max}\}$}
{
    \tcp*[h]{\textit{update discriminators (line 1 to 19)}}\;
    \For{$iter$ in $\{1,2,\cdots,T_d\}$}
    {
        $\boldsymbol{x}^B=[\boldsymbol{x}_1^B,\boldsymbol{x}_2^B]\subset X$ \tcp*[h]{\textit{Subsample a mini-batch of data}};\\
        
        Generate $\boldsymbol{z}^B$ with $z \sim \mathcal{N}(0,1)^l$ \\
        $\tilde{\boldsymbol{x}}^B = [\tilde{\boldsymbol{x}}_1^B, \tilde{\boldsymbol{x}}_2^B] = [G_1(\boldsymbol{z}^B), G_2(\boldsymbol{z}^B)] $ \tcp*[h]{\textit{Generate synthetic data}};\\
        \tcp*[h]{\textit{Compute losses of $D_1$, $D_2$, and $D_s$ (line 5 to 7)}}\;
        
        $
        \begin{aligned}
        \mathcal{L}_{D_1} &= -\mathbb{E}[D_1(\boldsymbol{x}_1^B)]+\mathbb{E}[D_1(\tilde{\boldsymbol{x}}_1^B)] \\
        & + \lambda \mathbb{E}\left[\left(\|\nabla D_1(\alpha^B \boldsymbol{x}_1^B+(1-\alpha^B) \tilde{\boldsymbol{x}}_1^B)\|_2-1\right)^2\right]
        \end{aligned} 
        $\\

        $
        \begin{aligned}
        \mathcal{L}_{D_2} &= -\mathbb{E}[D_2(\boldsymbol{x}_2^B)]+\mathbb{E}[D_2(\tilde{\boldsymbol{x}}_2^B)] \\
        & + \lambda \mathbb{E}\left[\left(\|\nabla D_2(\alpha^B \boldsymbol{x}_2^B+(1-\alpha^B) \tilde{\boldsymbol{x}}_2^B)\|_2-1\right)^2\right]
        \end{aligned} 
        $\\

        $
        \begin{aligned}
        \mathcal{L}_{D_s} &= -\mathbb{E}[D_s([\boldsymbol{f}_1,\boldsymbol{f}_2])]+\mathbb{E}[D_s([\tilde{\boldsymbol{f}}_1,\tilde{\boldsymbol{f}}_2])] \\
        & + \lambda \mathbb{E}\left[\left(\|\nabla D_s(\alpha^B [\boldsymbol{f}_1,\boldsymbol{f}_2]+(1-\alpha^B) [\tilde{\boldsymbol{f}}_1,\tilde{\boldsymbol{f}}_2])\|_2-1\right)^2\right] \\
        \end{aligned} 
        $\\
        
        $\mathcal{G}_{D_1} = \nabla_{\boldsymbol{\theta}_{D_1}} \mathcal{L}_{D_1} + \nabla_{\boldsymbol{\theta}_{D_1}}\mathcal{L}_{D_s}$ \tcp*[h]{\textit{Compute gradients of $D_1$}};\\
        $\mathcal{G}_{D_2} = \nabla_{\boldsymbol{\theta}_{D_2}} \mathcal{L}_{D_2} + \nabla_{\boldsymbol{\theta}_{D_2}}\mathcal{L}_{D_s}$ \tcp*[h]{\textit{Compute gradients of $D_2$}};\\
        $\mathcal{G}_{D_s} = \nabla_{\boldsymbol{\theta}_{D_s}}\mathcal{L}_{D_s}$ \tcp*[h]{\textit{Compute gradients of $D_s$}};\\

        $\boldsymbol{\theta}_{D_s} = \boldsymbol{\theta}_{D_s}-\eta_{D_s}\mathcal{G}_{D_s}$ \tcp*[h]{\textit{Update parameters of $D_s$}};\\
        
        \tcp*[h]{\textit{Update parameters of the second to the last layers of $D_1$}}\;
        $\boldsymbol{\theta}_{D_1}^{2:n_{D_1}} = \boldsymbol{\theta}_{D_1}^{2:n_{D_1}}-\eta_{D_1}\mathcal{G}_{D_1}^{2:n_{D_1}}$ \\
        
        \tcp*[h]{\textit{Update parameters of the second to the last layers of $D_2$}}\;
        $\boldsymbol{\theta}_{D_2}^{2:n_{D_2}} = \boldsymbol{\theta}_{D_2}^{2:n_{D_2}}-\eta_{D_2}\mathcal{G}_{D_2}^{2:n_{D_2}} $\\
        
        \tcp*[h]{\textit{Update parameters of the first layer of $D_1$ and $D_2$ (line 14 to 19)}}\; 
        \uIf{Training a differentially private version}{
        $\boldsymbol{\theta}_{D_1}^1 = \boldsymbol{\theta}_{D_1}^1-\eta_{D_1}(clip(\mathcal{G}_{D_1}^1,C)+\mathcal{N}\left(0, \sigma^2 (2C)^2 I\right))$ \\
        $\boldsymbol{\theta}_{D_2}^1 = \boldsymbol{\theta}_{D_2}^1-\eta_{D_2}(clip(\mathcal{G}_{D_2}^1,C)+\mathcal{N}\left(0, \sigma^2 (2C)^2 I\right))$} 
        \uElse{
        $\boldsymbol{\theta}_{D_1}^1 = \boldsymbol{\theta}_{D_1}^1-\eta_{D_1}\mathcal{G}_{D_1}^1 $\\
        $\boldsymbol{\theta}_{D_2}^1 = \boldsymbol{\theta}_{D_2}^1-\eta_{D_2}\mathcal{G}_{D_2}^1$} 
    }
        
    \tcp*[h]{\textit{update generators (line 20 to 25)}}\;
    
    $\tilde{\boldsymbol{x}}^B = [G_1(\boldsymbol{z}^B), G_2(\boldsymbol{z}^B)]$ \tcp*[h]{\textit{Generate a mini-batch of fake data $\tilde{\boldsymbol{x}}^B$}};\\
    
    $\mathcal{L}_{G_1} = -\mathbb{E}[D_1(\tilde{\boldsymbol{x}}_1)]-\mathbb{E}[D_s([\tilde{\boldsymbol{f}}_1,\tilde{\boldsymbol{f}}_2])]$ \tcp*[h]{\textit{Compute losses of $G_1$}};\\
    $\mathcal{L}_{G_2} = -\mathbb{E}[D_2(\tilde{\boldsymbol{x}}_2)]-\mathbb{E}[D_s([\tilde{\boldsymbol{f}}_1,\tilde{\boldsymbol{f}}_2])]$ \tcp*[h]{\textit{Compute losses of $G_2$}};\\
    where $\tilde{\boldsymbol{f}}_i$ is the intermediate feature when the input of $D_i$ is $\tilde{\boldsymbol{x}}_i$ \\
    
    \tcp*[h]{\textit{Compute gradients of $G_1$ and $G_2$}}\;
    $\mathcal{G}_{G_1} = \nabla_{\boldsymbol{\theta}_{G_1}} \mathcal{L}_{G_1}$, $\quad$
    $\mathcal{G}_{G_2} = \nabla_{\boldsymbol{\theta}_{G_2}} \mathcal{L}_{G_2}$ 
        
    \tcp*[h]{\textit{Update parameters of $G_1$ and $G_2$}}\;
    $\boldsymbol{\theta}_{G_1} = \boldsymbol{\theta}_{G_1}-\eta_{G_1}\mathcal{G}_{G_1},\quad \boldsymbol{\theta}_{G_2} = \boldsymbol{\theta}_{G_2}-\eta_{G_2}\mathcal{G}_{G_2}$ 
}
 Return $G_1$ and $G_2$.
\end{algorithmic}
\end{algorithm}

\subsection{Differentially Private VFLGAN} \label{sec:Differentially Private VFLGAN}
The training process of DP-VFLGAN is also summarized in Algorithm \ref{alg:Training DP-VFL-GAN} since the training processes of VFLGAN and DP-VFLGAN are the same for most steps. There are mainly two differences between the two training processes. (i) We apply the proposed Gaussian mechanism \eqref{eq:add noise} in DP-VFLGAN. When updating the parameters of the first linear layers of $D_1$ and $D_2$, we clip and add Gaussian noise to the gradients as follows,
\begin{align}
\label{eq:clip}
    &clip(\mathcal{G}_{D_i}^1,C) = \mathcal{G}_{D_i}^1 / \max \left(1,\left\|\mathcal{G}_{D_i}^1\right\|_2 / C\right) \quad i\in \{1,2\}, \\
\label{eq:add noise}
    &\mathcal{G}_{D_i}^1 = clip(\mathcal{G}_{D_i}^1,C) + \mathcal{N}\left(0, \sigma^2 (2C)^2 I\right),
\end{align}
where $\mathcal{G}_{D_i}^1$ denotes the gradients of the first layer parameters of $D_i$ and $C$ denotes the clipping bound. Previous Gaussian mechanisms for GANs \cite{Xue01, xie2018differentially, zhang2018differentially} clip and add Gaussian noise to all discriminator parameters. However, we proved in Theorem \ref{theorem RDP guarantee} that clipping and adding Gaussian noise to the gradients of the first-layer parameters of $D_1$ and $D_2$ can provide a DP guarantee for $G_1$ and $G_2$. (ii) We need to set a privacy budget before training DP-VFLGAN to ensure $G_1$ and $G_2$ satisfy $(\epsilon, \delta)$-DP. Here, we apply the official implementation of \cite{wang2019subsampled} to select the proper $T_{max}$ and $\sigma$ to achieve the privacy budget, i.e., $(\epsilon, \delta)$-DP, and the details are described in Section \ref{sec:RDP Guarantee}.

\subsection{RDP Guarantee and Proof} \label{sec:RDP Guarantee}
\begin{myTheo} \label{theorem RDP guarantee}
    (RDP Guarantee) All the local discriminators and generators satisfy ($\alpha$, $\alpha/(2\sigma^2)$)-RDP in one training iteration of DP-VFLGAN.
\end{myTheo}
\begin{proof}
The proof of Theorem \ref{theorem RDP guarantee} is presented in \textbf{Appendix} \ref{sec:proof of RDP guarantee}.
\end{proof}

Now, we introduce how to select proper $\sigma$ and $T_{max}$ to meet our DP budget using Theorem \ref{theorem RDP guarantee}. First, the RDP guarantee in Theorem \ref{theorem RDP guarantee} can be enhanced by Proposition \ref{theorem: subsample RDP} for external attackers and the server since we subsample mini-batch records from the whole training dataset. Then, RDP budget is accumulated by $T_{max}$ iterations, which can be calculated according to Proposition \ref{Prop: RDP composition}. Last, the RDP guarantee is converted to DP guarantee using \eqref{eq:RDP to DP}. We can adjust the $\sigma$ and $T_{max}$ to make the calculated DP guarantee meet our DP budget. 
Notably, similar to the DP guarantee in \cite{Xue01}, the above DP guarantee specifically addresses external threats. This is due to the deterministic nature of mini-batch selection in each training iteration for internal parties (excluding the server), as opposed to a subsampling process. For internal adversaries, privacy can be enhanced by sacrificing efficiency. For example, the mini-batch size can be changed to $\hat{B}>B$. When updating the parameters, each party can randomly select the gradients of $B$ samples and mask the gradients of other ($\hat{B}-B$) samples. In this way, all parties do not know precisely which samples are used by others to update the parameters. 

\subsection{Differences between VFLGAN and VertiGAN}
As discussed in Section \ref{sec:vertigan}, in VertiGAN, the objective functions \eqref{eq:obj_wgan_gp} of the local discriminator and generator in each party are the same as WGAN\_GP's objective functions \eqref{eq:min-max objective}. According to \cite{gulrajani2017improved}, this can ensure that the synthetic data ($\tilde{\boldsymbol{x}}_i$) generated by part $i$ follows a similar distribution of the real data ($\boldsymbol{x}_i$) of party $i$, i.e., $P_{\tilde{\boldsymbol{x}}_i} \simeq P_{{\boldsymbol{x}}_i}$. After calculating the gradients of each local generator backbone, VertiGAN applies HFL \eqref{eq:hfl} to maintain the same generator backbone across all parties and assumes that the same generator backbones can learn the correlations among attributes across those parties, i.e., $P_{[\tilde{\boldsymbol{x}}_1, \cdots, \tilde{\boldsymbol{x}}_N]} \simeq P_{[\tilde{\boldsymbol{x}}_1, \cdots, \tilde{\boldsymbol{x}}_N]}$, which is experimentally proved less effective in this paper. 
On the other hand, VFLGAN incorporates the objective functions of WGAN\_GP into the VFL framework with a shared discriminator. Accordingly, we propose novel objective functions, \eqref{eq:obj_system_model} and \eqref{eq:obj_system_model_2}, for VFLGAN. From \eqref{eq:obj_system_model} and \eqref{eq:obj_system_model_2}, we can see that the objective functions of VFLGAN involve all parties through the shared discriminator while the objective functions of each party are independent in VertiGAN. In \eqref{eq:obj_system_model} and \eqref{eq:obj_system_model_2}, according to \cite{gulrajani2017improved}, the fractions involving local discriminatory $D_i$ can ensure that $P_{\tilde{\boldsymbol{x}}_i} \simeq P_{{\boldsymbol{x}}_i}$ and the fractions involving the shared discriminator $D_s$ can ensure that $P_{[\tilde{\boldsymbol{x}}_1, \cdots, \tilde{\boldsymbol{x}}_N]} \simeq P_{[\tilde{\boldsymbol{x}}_1, \cdots, \tilde{\boldsymbol{x}}_N]}$. 

Another difference is that the differentially private VertiGAN applies the same Gaussian mechanism as \cite{xie2018differentially, zhang2018differentially}, which clips and adds noise to all gradients. However, as proved in Thereom \ref{theorem RDP guarantee}, clipping and adding noise to the gradients of the first linear layer of the local discriminator is enough to provide a DP guarantee to all local discriminators and local generators in DP-VFLGAN. Based on this discovery, DP-VFLGAN applies a new variant of the Gaussian mechanism as discussed in Section \ref{sec:Differentially Private VFLGAN}.

\section{Privacy Leakage Measurement} \label{sec:Privacy Leakage Measurement}

In this section, we introduce two distinct auditing schemes, i.e., the Auditing Scheme for Synthetic Datasets (ASSD) and the Auditing Scheme for Intermediate Features (ASIF). These schemes are designed to conduct MI attacks within a leave-one-out setting \cite{Jiayuan01}, a scenario in which the adversary is presumed to be aware of the entire dataset except for the target record \((x_t)\). While stringent for an external adversary, this assumption provides an empirical upper limit on the success rates of MI attacks. Additionally, this setting is deemed realistic for data publishers who can readily manipulate the inclusion of the target record in the training process.
The differentiation between our MI attack strategy and that described in \cite{Jiayuan01} hinges on two primary aspects: (i) The focus of our MI attack is on synthetic datasets, as opposed to the machine learning models targeted in \cite{Jiayuan01}; (ii) Our auditing schemes adopt the shadow model attack method \cite{Theresa01, Reza01, houssiau01} as the adversary model, contrasting with the binary hypothesis test approach utilized in \cite{Jiayuan01}, which is specifically crafted for classification models.

\subsection{Auditing Scheme for Synthetic Datasets} \label{sec:Auditing Scheme for Synthetic Datasets}

The following privacy game can define the MI attack applied in ASSD. \textbf{(i)} First, the challenger selects a fixed target record $\boldsymbol{x}_t$ from a given training dataset $X$. \textbf{(ii)} Then, the challenger trains a generator $G_0 \stackrel{s_{0}}{\longleftarrow} \mathcal{T}(X \setminus \boldsymbol{x}_t)$ on $X \setminus \boldsymbol{x}_t$ (dataset $X$ excluding $\boldsymbol{x}_t$) by using a fresh random seed $s_{0}$ in the training algorithm $\mathcal{T}$. In this paper, $\mathcal{T}$ refers to Algorithm \ref{alg:Training DP-VFL-GAN}. \textbf{(iii)} The challenger trains another generator $G_1 \stackrel{s_{1}}{\longleftarrow} \mathcal{T}(X)$ on $X$ by using a fresh random seed $s_{1}$. 
\textbf{(iv)} The challenger applies the two generators, $G_0$ and $G_1$, to produce two synthetic datasets, $\tilde{X}_1$ and $\tilde{X}_2$.
\textbf{(v)} The challenger flips a random unbiased coin $b\in \{0,1\}$ and sends the synthetic dataset and target record \{$\tilde{X}_b, \boldsymbol{x}_t$\} to the adversary. \textbf{(vi)} The adversary tries to figure out the true $b$ based on the observation of $\{\tilde{X}_b, \boldsymbol{x}_t\}$, i.e., $\hat{b} \leftarrow \mathcal{A}(\tilde{X}_b, \boldsymbol{x}_t)$. \textbf{(vii)} If $\hat{b}=b$, the adversary wins. Otherwise, the challenger wins. This MI attack is summarized in Algorithm \ref{Game:MI}. 

We apply the shadow modelling approach \cite{Reza01} to train the adversary model $\mathcal{A}$ in Algorithm \ref{Game:MI} through the following process. \textbf{(i)} Train $M$ generators, $G_{0_{1:M}}$, on dataset $X\setminus \boldsymbol{x}_t$ by using different random seeds $s_{0_{1:M}}$ in the training algorithm $\mathcal{T}$. \textbf{(ii)} Apply $G_{0_{1:M}}$ to generate $M$ synthetic datasets $\tilde{X}_{0_{1:M}}$. \textbf{(iii)} Train $M$ generators, $G_{1_{1:M}}$, on dataset $X$ by using different random seeds $s_{1_{1:M}}$. \textbf{(vi)} Apply $G_{1_{1:M}}$ to generate $M$ synthetic datasets $\tilde{X}_{1_{1:M}}$. \textbf{(v)} Extract features of the synthetic datasets $\tilde{X}_{0_{1:M}}$ and $\tilde{X}_{1_{1:M}}$. \textbf{(vi)} Train the adversary model $\mathcal{A}$ to distinguish whether the features are from $\tilde{X}_{0_{1:M}}$ or $\tilde{X}_{1_{1:M}}$. The training process is summarized in Algorithm \ref{alg:MI adversary}. 

In summary, the adversary is trained to detect whether the target exists in the training data through the features of the synthetic dataset. Then, we evaluate the capability of the adversary through the MI attack. The success rate of the MI attack can reflect the potential privacy leakage that external attackers can take advantage of through the synthetic dataset and MI attack.

\begin{algorithm}[htbp]
\footnotesize
\caption{Membership Inference Attack}\label{Game:MI}
\begin{algorithmic}[1]
\REQUIRE Training algorithm $\mathcal{T}$; dataset $X$; target record $\boldsymbol{x}_t$; unbiased coin $b$; fresh random seeds $s_{0}$ and $s_{1}$; generators of VFLGAN $G_i$; synthetic dataset $\tilde X$; Gaussian noise $\boldsymbol{z}$.
\ENSURE Success or failure.
\noindent{\rule{0.95\linewidth}{0.4pt}}
\STATE $G_0 \stackrel{s_{0}}{\longleftarrow} \mathcal{T}(X \setminus \boldsymbol{x}_t)$
\STATE $G_1 \stackrel{s_{1}}{\longleftarrow} \mathcal{T}(X)$
\STATE $\tilde{X}_0 \longleftarrow G_0(\boldsymbol{z})$ \& $\tilde{X}_1 \longleftarrow G_1(\boldsymbol{z})$
\STATE $b \sim \{0,1\}$
\STATE $\hat{b} \longleftarrow \mathcal{A}(\tilde{X}_b, \boldsymbol{x}_t)$
\IF{$\hat{b}==b$}
\STATE Output success
\ELSE
\STATE Output failure
\ENDIF
\end{algorithmic}
\end{algorithm}

\begin{algorithm}[htbp]
\footnotesize
\caption{Training the Adversary of MI Attack} \label{alg:MI adversary}
\begin{algorithmic}[1]
\REQUIRE Training algorithms $\mathcal{T}$ and $\mathcal{T}_\mathcal{A}$; dataset $X$; target record $\boldsymbol{x}_t$; random seeds $s_{0_{1:M}}$ and $s_{1_{1:M}}$; synthetic dataset $\tilde X$; feature extraction function $Extr(\cdot)$.
\ENSURE Trained $\mathcal{A}$.
\noindent{\rule{0.95\linewidth}{0.4pt}}
\STATE $G_{0_{1:M}} \stackrel{s_{0_{1:M}}}{\longleftarrow} \mathcal{T}(X \setminus \boldsymbol{x}_t)$
\STATE $\tilde{X}_{0_{1:M}} \longleftarrow G_{0_{1:M}}$
\STATE $G_{1_{1:M}} \stackrel{s_{1_{1:M}}}{\longleftarrow} \mathcal{T}(X)$
\STATE $\tilde{X}_{1_{1:M}} \longleftarrow G_{1_{1:M}}$
\STATE $Feat_{0_{1:M}} \longleftarrow Extr(\tilde{X}_{0_{1:M}})$ \& $Feat_{1_{1:M}} \longleftarrow Extr(\tilde{X}_{1_{1:M}})$
\STATE $\mathcal{A} \longleftarrow \mathcal{T}_\mathcal{A}(Feat_{0_{1:M}}, Feat_{1_{1:M}})$
\end{algorithmic}
\end{algorithm}

\subsection{Auditing Scheme for Intermediate Features} \label{sec:Auditing Scheme for Intermediate Features}
The MI attack and training process of the adversary applied in ASIF are very similar to those (Algorithm \ref{Game:MI} and Algorithm \ref{alg:MI adversary}) of ASSD. So we briefly introduce ASIF here and details can be found in Appendix \ref{app:Auditing Scheme for Intermediate Features}. In ASIF, the challenger also trains two VFLGANs ($\theta_0$ and $\theta_1$) on datasets $X$ and $X\setminus \boldsymbol{x}_t$, respectively, and uses their first part of local discriminators to generate intermediate features. The adversary tries to figure out whether the given intermediate feature is from $\theta_0$ or $\theta_1$. Similar to ASSD, ASIF applies shadow model attack as the adversary model. The difference is that the adversary model of ASSD is to figure out whether the target record $\boldsymbol{x}_t$ appears in the training data through a given synthetic dataset, while the adversary model of ASIF is to do the same job through given intermediate features.

\section{Experimental Details and Results}

In this section, we commence by detailing our experimental setup, including the datasets, baseline methods, and evaluation metrics. Following this, we rigorously assess the effectiveness of VFLGAN and its privacy-enhanced counterpart, DP-VFLGAN. The section concludes with an in-depth analysis of the privacy leakage associated with VFLGAN and DP-VFLGAN, employing our newly developed privacy auditing schemes, ASSD and ASIF.

\subsection{Experiment Setup}
\subsubsection{Datasets}

{We employ three datasets for evaluation: the MNIST dataset, Adult dataset \cite{misc_adult_2}, Wine Quality datasets \cite{misc_wine_quality_186}, Credit dataset \cite{misc_south_german_credit_522}, and HCV dataset \cite{misc_hepatitis_c_virus_for_egyptian_patients_503}.} 
All datasets utilized in this study are derived from real-world sources. To simulate the scenario where two distinct parties possess different attributes of the same group of data samples, we divided each dataset into two sub-datasets through a vertical split.

\textbf{MNIST} is a widely recognized dataset comprised of handwriting digit images. In our experiment, each image is transformed into a one-dimensional vector with 784 attributes. These attributes are split between two clients in a vertically partitioned manner: the first 392 attributes to Client 1, and the remaining 392 attributes to Client 2. Each attribute is an integer ranging from 0 to 255. Although image data is not typically partitioned vertically, we utilize this to demonstrate the shortcomings of VertiGAN \cite{Xue01}, which our proposed VFLGAN addresses effectively.

\textbf{Adult} \cite{misc_adult_2} is a prominent machine-learning dataset used for classification tasks. It consists of fifteen attributes: fourteen features describing various aspects of an individual, and a label indicating whether the person's income exceeds 50,000 dollars. The dataset comprises six integer attributes and nine categorical attributes.

\textbf{Wine quality} \cite{misc_wine_quality_186} is used for classification tasks and includes two subsets: Red-Wine-Quality and White-Wine-Quality. Each subset contains twelve attributes, with eleven continuous attributes serving as features of a wine sample and one categorical attribute denoting the wine's quality.

\textbf{Credit} \cite{misc_south_german_credit_522} is a financial dataset used for classification tasks. It consists of 21 attributes: 20 features describing various aspects of an individual and a label indicating the person's credit. The dataset comprises 3 integer attributes and 18 categorical attributes.

\textbf{HCV} \cite{misc_hepatitis_c_virus_for_egyptian_patients_503} is a medical dataset used for classification tasks. It consists of 29 attributes: 28 features describing personal information and medicine check results and a label indicating the HCV staging. The dataset comprises 13 integer attributes and 16 categorical attributes.

\textbf{Preprocessing}. Same as \cite{gulrajani2017improved}, the MNIST images undergo normalization, with all attributes scaled to the range of $[-1,1]$. We adopt a different strategy for the tabular datasets from that used in \cite{Xue01}. Instead of transforming all attributes to binary form that can cause information loss and potential privacy breach \cite{Theresa01}, we convert categorical attributes into one-hot vectors while integer and continuous attributes are standardized using the following formula:
\begin{equation}
\label{eq:normalization}
    a = (a-\mu_a)/\sigma_a,
\end{equation}
where $\mu_a$ and $\sigma_a$ represent the mean and variance of the attribute $a$, respectively.
The main architectures of generators for different datasets are similar but the last activation layers vary according to the different preprocessing. For generators trained on the MNIST dataset, we utilize the $\operatorname{tanh}(\cdot)$ function as the final activation layer. This choice aligns with the fact that all standardized attributes range from -1 to 1. In contrast, for generators trained on tabular datasets, the approach differs. Here, the identity function is applied as the final activation layer for normalized integer and continuous attributes, acknowledging that these do not have a fixed range. For one-hot attributes, we implement the Gumbel softmax function, as detailed in \cite{jang2016categorical}.

\subsubsection{Baseline models}
First of all, we employ WGAN\_GP \cite{gulrajani2017improved} as a baseline model, which is trained in a centralized manner, to establish a performance upper bound for models trained in a vertically partitioned manner. Next, we select VertiGAN and DPLT as comparative models to underscore the superiority of VFLGAN in vertically partitioned scenarios. Another baseline is a modified version of VFLGAN, named VFLGAN-base, which omits the second part of $D_1$ and $D_2$, allowing us to assess the impact of these modules. Readers can refer to Appendix \ref{app:Details of VFLGAN-base} for details of VFLGAN-base. 

We apply WGAN\_GP to evaluate the proposed Gaussian mechanism and choose GS-WGAN \cite{chen2020gs} and DPSGD GANs \cite{xie2018differentially, zhang2018differentially} as baselines. This is due to their inability to provide a DP guarantee for our proposed VFLGAN. Notably, we maintain consistent generator and discriminator architectures across various differentially private mechanisms for a fair comparison. For comprehensive information on the architectures of the proposed VFLGAN and the baselines, please see \textbf{Appendix} \ref{sec:Architectures}. Furthermore, we apply Proposition \ref{theorem: subsample RDP} to calculate the privacy budget for DPSGD GAN, ensuring a fair comparison by providing a tighter DP bound. This proposition is employed by both GS-WGAN and our mechanism for calculating the DP budget. We exclude PATE-GAN \cite{jordon2018pate} in our baselines, as its conceptual framework is akin to GS-WGAN but presents challenges in training. Specifically, the authors of GS-WGAN \cite{chen2020gs} reported difficulties in training PATE-GAN on the MNIST dataset.

\subsubsection{Utility Metrics}\label{sec:utility metrics} According to \cite{Xue01}, there are two primary metrics for evaluating the utility of synthetic datasets: statistical similarity and AI training performance.
For statistical similarity, we initially consider the Average Variant Distance (AVD) method, the same as \cite{Xue01,8667703}, which measures discrepancies between real and synthetic datasets. However, AVD is primarily applicable to discrete attributes, making it less suitable for our study. Instead, we utilize the Fréchet Distance (FD) \cite{frechet1957distance}, inspired by the Fréchet Inception Distance (FID) \cite{heusel2017gans}. A lower FD indicates a closer resemblance between the real and synthetic datasets. We directly compute the FD using raw data for tabular datasets. However, individual image pixels lack meaningful semantic content, so for the MNIST dataset, we employ an Autoencoder to derive more semantically rich latent features, which are then used to calculate FD. The FD is calculated as,
 \begin{equation}
    \label{eq:FD}
    FD(X,\tilde X) = \left\|\boldsymbol{\mu}-\boldsymbol{\tilde\mu}\right\|_2^2+\operatorname{Tr}\left(\boldsymbol{V}+\boldsymbol{\tilde{V}}-2\left(\boldsymbol{{V}}\boldsymbol{\tilde{V}}\right)^{1 / 2}\right),
\end{equation}
where $\operatorname{Tr}(\cdot)$ calculates the trace of the input matrix, $\boldsymbol{\mu}$ and $\boldsymbol{V}$ denote the mean and variance of the real dataset $X$, and $\boldsymbol{\tilde\mu}$ and $\boldsymbol{\tilde{V}}$ are those of the synthetic dataset $\tilde X$.
We evaluate AI training performance using accuracy and F1 scores in four different settings: training and testing on real data (TRTR), on synthetic data (TSTS), training on real data and testing on synthetic data (TRTS), and vice versa (TSTR). The TRTR setting serves as a baseline for AI model performance. The similarity of metrics in TSTS, TRTS, and TSTR settings to TRTR indicates how closely the synthetic data distribution resembles real data distribution. Notably, higher accuracy and F1 scores in TSTS, TRTS, and TSTR settings do not necessarily imply higher synthetic data quality.
For the unlabeled synthetic MNIST data, we use the Inception Score (IS) \cite{barratt2018note}, which assesses how convincing the synthetic images are in belonging to real labels. IS is calculated as:
\begin{equation}
    \label{eq:IS}
    \operatorname{IS}(G)=\exp \left(\mathbb{E}_{\boldsymbol{\tilde x}} D_{K L}(p(y \mid \boldsymbol{\tilde x}) \| p(y))\right),
\end{equation}
where $p(y \mid \boldsymbol{\tilde x})$ ) is the conditional class distribution from the Inception model \cite{szegedy2016rethinking}, and $p(y)=\int_{\boldsymbol{\tilde x}} p(y \mid \boldsymbol{\tilde x}) p_g(\boldsymbol{\tilde x})$ is the marginal class distribution.
The IS measures the distribution distance between $p(y)$ and $p(y \mid \boldsymbol{\tilde x})$. A higher IS indicates a more realistic generation of synthetic data. Since the Inception model is not optimized for MNIST, we replace it with an MLP model trained for digit classification on the MNIST dataset.

\subsection{Utility Results of VFLGAN}
In this section, we evaluate the performance of VFLGAN using the above-mentioned six datasets, chosen for their relevance and representativeness in diverse scenarios. 
We delve into a detailed comparison of VFLGAN against established benchmarks, showcasing its enhanced ability to capture intricate correlations among attributes distributed across different parties. Additionally, we employ a suite of robust evaluation metrics, carefully selected to provide a multi-faceted view of VFLGAN's performance. These metrics are designed to assess the distribution similarity between synthetic datasets and real datasets and the applicability of synthetic datasets in downstream tasks, thereby offering a holistic understanding of the model's capabilities and superiority in the realm of synthetic data generation.

\subsubsection{MNIST Dataset}
All models were trained for 300 epochs on the MNIST dataset, and the FD curves, as shown in the left part of Fig. \ref{fig:FD_and_IS}, reveal insightful trends. WGAN\_GP emerges as the top performer in terms of FD, with VFLGAN closely following. VFLGAN-base also surpasses VertiGAN, indicating a more effective generation of vertically partitioned data through VFL compared to HFL.
For a practical demonstration, we selected models with the lowest FD across training epochs to generate synthetic images. The outcomes, as illustrated in Fig. \ref{fig:MNIST_demo}, reveal a marked difference in the quality of the generated images. 
Specifically, the synthetic dataset generated by VFLGAN has an FD score approximately twice as high as the synthetic dataset produced by WGAN\_GP, while VertiGAN's synthetic dataset has an FD score nearly seven times higher. This gap underscores the substantial advantage of VFLGAN over VertiGAN.  
Moreover, about half of VertiGAN's synthetic samples are either unrecognizable or lack continuity. On the other hand, while the VFLGAN-base produces more discernible images, some noise is evident, likely a result of information loss inherent to the VFL framework. To address this, we integrated an additional component into $D_1$ and $D_2$, i.e., the second part of $D_1$ and $D_2$ and corresponding loss functions, to enhance the similarity of the generated  ${\tilde{\boldsymbol{x}}}_1$ and ${\tilde{\boldsymbol{x}}}_2$ to the real samples $\boldsymbol{x}_1$ and $\boldsymbol{x}_2$, as per Equations \eqref{eq:gradient_d1} and \eqref{eq:gradient_d2}. This modification is evident in the improved clarity and recognizability of the synthetic data generated by VFLGAN. The right side of Fig. \ref{fig:FD_and_IS} illustrates the IS curves during training. The IS for real images sets a benchmark for synthetic ones. Consistent with the FD results, WGAN\_GP leads in terms of generative model performance, with VFLGAN coming in third but maintaining a significant lead over VertiGAN. The IS performance of VFLGAN-base also exceeds that of VertiGAN. Ultimately, the alignment of performance rankings between IS and FD further affirms the superiority of VFLGAN in generating high-quality synthetic data.

\begin{figure}[htbp]
\center{\includegraphics[width=\linewidth, scale=1.]{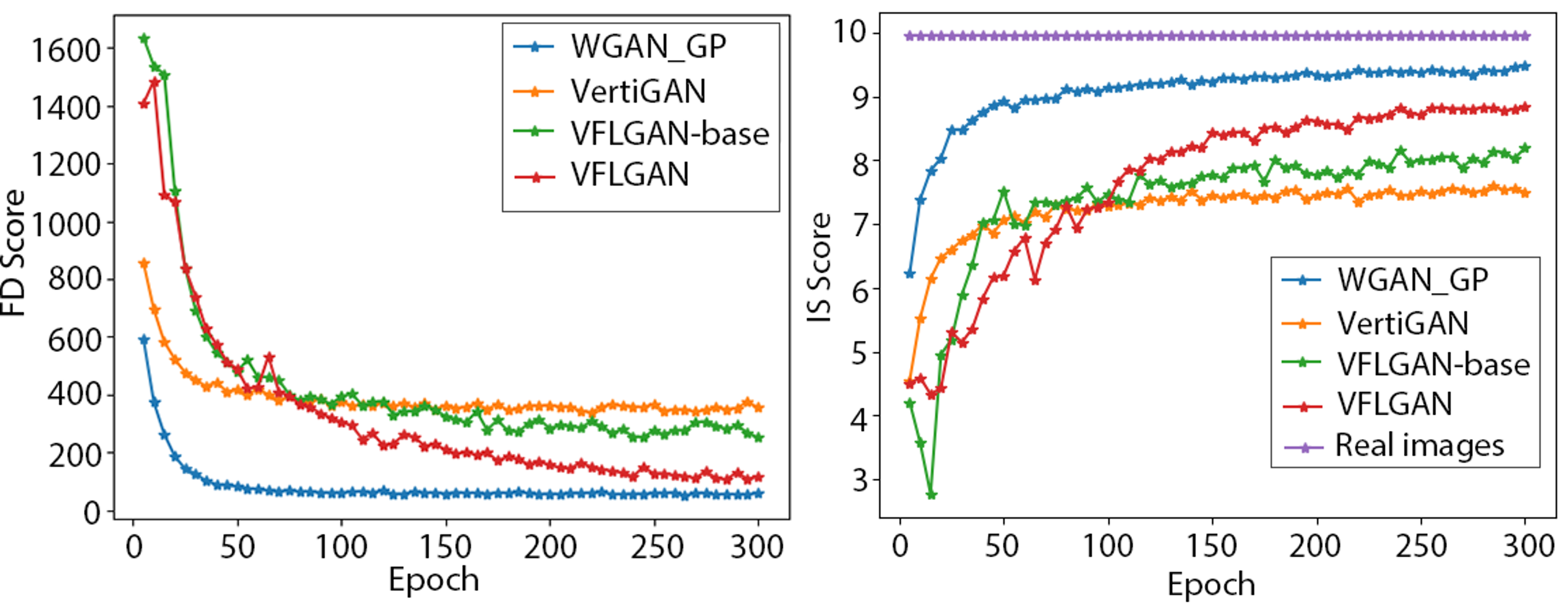}}
\caption{FD curves (lower is better) and IS curves (higher is better) on the MNIST dataset.}
\label{fig:FD_and_IS}
\end{figure}

\begin{figure}[htbp]
\center{\includegraphics[width=0.7\linewidth, scale=1.]{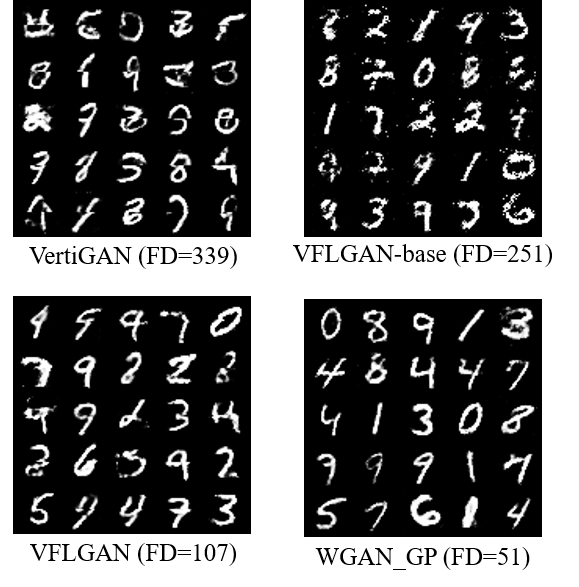}}
\caption{Synthetic digit samples.}
\label{fig:MNIST_demo}
\end{figure}

\begin{figure*}[htbp]
\center{\includegraphics[width=\linewidth, scale=1.]{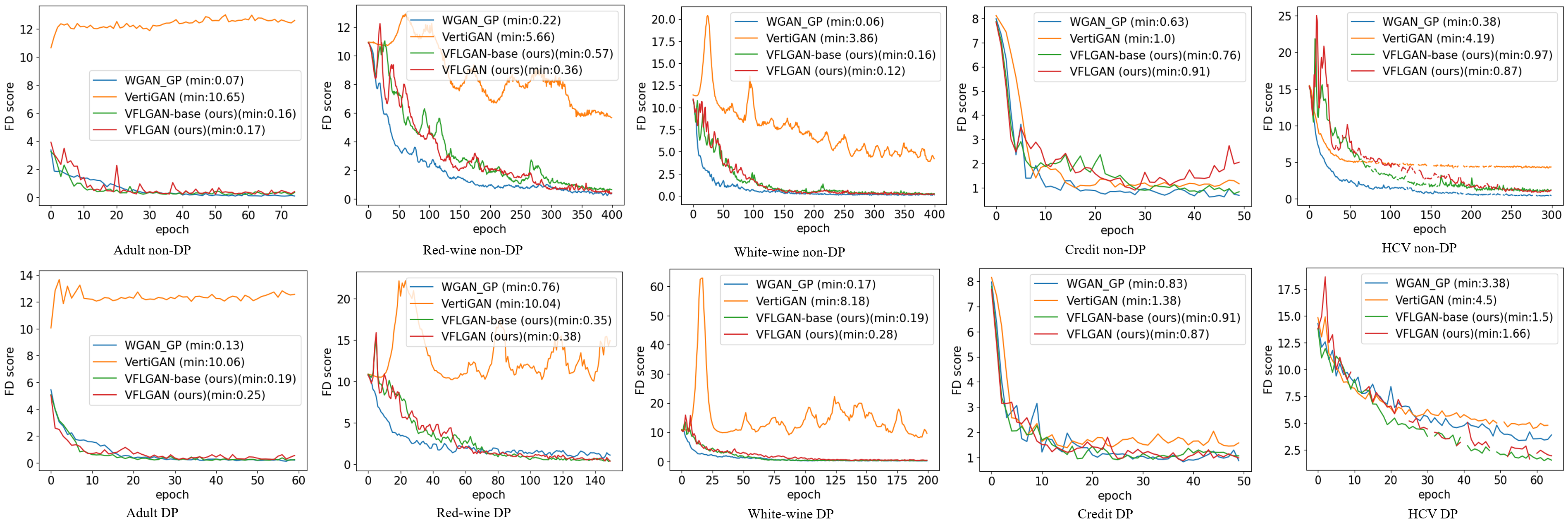}}
\caption{FD curves of different methods during training on the five datasets. The upper five figures show the FD curves of non-DP models. The lower five figures show the FD curves of DP models satisfying (10,$\delta$)-DP. Minimum values on each curve are shown by legends. The discontinuity in the FD curves on HCV datasets is caused by nan value when calculating FD.}
\label{fig:fid_all_datasets}
\end{figure*}

\begin{table*}[htbp]
\centering
{
\renewcommand\arraystretch{.8}
    \caption{Classification Accuracy and F1 score of Random Forest Models to Evaluate the AI Training Utility of the Synthetic Data on the Wine-Quality Datasets and Adult Dataset. Absolute different metric values compared with the TRTR setting are shown in the bracket and the main evaluation metric is `Total Difference', the \textcolor{red}{lower} the better.}
    \label{tab:utility classifiction}

\begin{center}
    \scalebox{0.8}{\begin{tabular}{ccccccccc|c}
    \hline
    {} & \multicolumn{2}{c}{\textbf{TRTR}} & \multicolumn{2}{c}{\textbf{TSTS}} & \multicolumn{2}{c}{\textbf{TRTS}} & \multicolumn{2}{c}{\textbf{TSTR}} & {}\\
    \textbf{Methods} & \textbf{Acc} & \textbf{F1} & \textbf{Acc} & \textbf{F1} & \textbf{Acc} & \textbf{F1} & \textbf{Acc} & \textbf{F1}  & {\textbf{Total Difference}} \\
     \hline
     \multicolumn{10}{c}{\textbf{Red-Wine-Quality}} \\
    \hline
     {\textbf{WGAN\_GP}\cite{gulrajani2017improved}} &0.59 &0.26 &0.57 (0.02) &0.28 (0.02) &0.55 (0.03) &0.25 (0.01) &0.58 (0.01) &0.26 (0.00) & \textbf{\textcolor{red}{0.09}}\\
     \hline
     \textbf{VertiGAN}\cite{Xue01} &0.59 &0.26 &0.53 (0.06) &0.21 (0.05) &0.51 (0.08) &0.19 (0.07) &0.39 (0.20) &0.14 (0.12) &0.58 \\  
    \hline
    \textbf{VFLGAN-base (ours)} &0.59 &0.26 &0.55 (0.04) &0.27 (0.01) &0.51 (0.08) &0.21 (0.05) &0.55 (0.04) &0.27 (0.01) & 0.23\\
     \hline
     \textbf{VFLGAN (ours)} &0.59 &0.26 &0.56 (0.03) &0.26 (0.00) &0.54 (0.05) &0.24 (0.02) &0.58 (0.01) &0.26 (0.00) & \textbf{\textcolor{blue}{0.11}}\\
     \hline
    \multicolumn{10}{c}{\textbf{White-Wine-Quality}} \\
    \hline
     {\textbf{WGAN\_GP}\cite{gulrajani2017improved}} &0.53 &0.21 &0.53 (0.00) &0.24 (0.03) &0.53 (0.00) &0.21 (0,00) &0.53 (0.00) &0.20 (0.01) & \textbf{\textcolor{red}{0.04}}\\
     \hline   
     \textbf{VertiGAN}\cite{Xue01} &0.53 &0.21 &0.53 (0.00) &0.16 (0.05) &0.48 (0.05) &0.16 (0.05) &0.47 (0.06) &0.14 (0.07) & 0.28\\  
    \hline
    \textbf{VFLGAN-base (ours)} &0.53 &0.21 &0.53 (0.00) &0.21 (0.00) &0.52 (0.01) &0.20 (0.01) &0.52 (0.01) &0.20 (0.01) & \textbf{\textcolor{red}{0.04}}\\
     \hline
     \textbf{VFLGAN (ours)} &0.53 &0.21 &0.51 (0.02) &0.20 (0.01) &0.50 (0.03) &0.19 (0.02) &0.51 (0.02) &0.19 (0.02) & 0.12\\
     \hline
    \multicolumn{10}{c}{\textbf{Adult Dataset}} \\
    \hline
    {\textbf{WGAN\_GP}\cite{gulrajani2017improved}} &0.82 &0.72 &0.82 (0.00) &0.74 (0.02) &0.81 (0.01) &0.71 (0.01) &0.72 (0.10) &0.69 (0.03) & \textbf{\textcolor{red}{0.17}} \\
     \hline
     \textbf{VertiGAN}\cite{Xue01} &0.82 &0.72 &0.59 (0.23) &0.37 (0.35) &0.53 (0.29) &0.48 (0.24) &0.75 (0.07) &0.43 (0.29) & 1.47\\  
    \hline
    \textbf{VFLGAN-base (ours)} &0.82 &0.72 &0.81 (0.01)&0.66 (0.06) &0.80 (0.02) &0.68 (0.04)&0.76 (0.06) &0.72 (0.00) & \textbf{\textcolor{blue}{0.19}}\\
     \hline
      \textbf{VFLGAN (ours)} &0.82 &0.72 &0.82 (0.00) &0.65 (0.07) &0.79 (0.03) &0.65 (0.07) &0.79 (0.03) &0.54 (0.18) & {0.38}\\  
     \hline
     
     \multicolumn{10}{c}{\textbf{Credit Dataset}} \\
    \hline
    {\textbf{WGAN\_GP}\cite{gulrajani2017improved}} &0.75 &0.67 &0.68 (0.07) &0.55 (0.12) &0.67 (0.08) &0.56 (0.11) &0.67 (0.08) &0.53 (0.14) & \textbf{\textcolor{blue}{0.60}} \\
     \hline
     \textbf{VertiGAN}\cite{Xue01}&0.75 &0.67  &0.67 (0.08) &0.47 (0.20) &0.64 (0.11) &0.46 (0.21) &0.68 (0.07) &0.42 (0.25) & 0.92\\  
    \hline
    \textbf{VFLGAN-base (ours)} &0.75 &0.67  &0.72 (0.03)&0.67 (0.00) &0.64 (0.11) &0.51 (0.16)&0.63 (0.12) &0.58 (0.09) & \textbf{\textcolor{red}{0.51}}\\
     \hline
      \textbf{VFLGAN (ours)} &0.75 &0.67 &0.77 (0.02) &0.45 (0.22) &0.69 (0.06) &0.48 (0.19) &0.70 (0.05) &0.42 (0.25) & {0.79}\\  
     \hline
     
     \multicolumn{10}{c}{\textbf{HCVEGY Dataset}} \\
    \hline
    {\textbf{WGAN\_GP}\cite{gulrajani2017improved}}&0.23 &0.20 &0.29 (0.06) &0.17 (0.03) &0.23 (0.00) &0.21 (0.01) &0.24 (0.01) &0.16 (0.04) & \textbf{\textcolor{red}{0.15}} \\
     \hline
     \textbf{VertiGAN}\cite{Xue01}  &0.23 &0.20 &0.30 (0.07)&0.21 (0.01) &0.25 (0.02) &0.22 (0.02)&0.24 (0.01) &0.14 (0.06) & \textbf{\textcolor{blue}{0.19}}\\
    \hline
    \textbf{VFLGAN-base (ours)} &0.23 &0.20 &0.31 (0.08)&0.26 (0.06) &0.24 (0.01) &0.20 (0.00)&0.27 (0.04) &0.17 (0.03) & {0.22}\\
     \hline
      \textbf{VFLGAN (ours)} &0.23 &0.20 &0.29 (0.06)&0.22 (0.02) &0.28 (0.05) &0.24 (0.04)&0.25 (0.02) &0.17 (0.03) & {0.22}\\
     \hline
     \multicolumn{10}{c}{\textbf{Red-Wine-Quality $(10, 5\times 10^{10-4})$-DP}} \\
     \hline
     {\textbf{WGAN\_GP}\cite{gulrajani2017improved}} &0.59 &0.26 &0.56 (0.03) &0.27 (0.01) &0.56 (0.03) &0.25 (0.01) &0.56 (0.03) &0.24 (0.02) & \textbf{\textcolor{red}{0.13}}\\
     \hline
     \textbf{VertiGAN}\cite{Xue01} &0.59 &0.26 &0.62 (0.03) &0.13 (0.13) &0.46 (0.13) &0.16 (0.10) &0.40 (0.19) &0.10 (0.16) &0.74 \\ 
     \hline
    \textbf{DPLT} &0.59 &0.26 &0.37 (0.22) &0.26 (0.00) &0.36 (0.23) &0.22 (0.04) &0.31 (0.28) &0.16 (0.10) & {0.87} \\ 
    \hline
    \textbf{VFLGAN-base (ours)} &0.59 &0.26 &0.59 (0.00) &0.23 (0.03) &0.59 (0.00) &0.23 (0.03) &0.57 (0.02) &0.21 (0.05) & \textbf{\textcolor{red}{0.13}}\\
     \hline
     \textbf{VFLGAN (ours)} &0.59 &0.26 &0.61 (0.02) &0.22 (0.04) &0.60 (0.01) &0.24 (0.02) &0.55 (0.04) &0.20 (0.06) & 0.19\\
     \hline
    \multicolumn{10}{c}{\textbf{White-Wine-Quality $(10, 2\times 10^{10-4})$-DP}} \\
    \hline
     {\textbf{WGAN\_GP}\cite{gulrajani2017improved}} &0.53 &0.21 &0.55 (0.02) &0.24 (0.03) &0.53 (0.00) &0.20 (0.01) &0.52 (0.01) &0.20 (0.01) & \textbf{\textcolor{red}{0.08}}\\
     \hline   
     \textbf{VertiGAN}\cite{Xue01} &0.53 &0.21 &0.70 (0.17) &0.22 (0.01) &0.48 (0.05) &0.17 (0.04) &0.38 (0.15) &0.12 (0.09) & 0.51\\  
     \hline
    \textbf{DPLT}  &0.53 &0.21 &0.43 (0.10) &0.15 (0.06) &0.42 (0.11) &0.17 (0.04) &0.49 (0.04) &0.14 (0.07) & {0.42} \\ 
    \hline
    \textbf{VFLGAN-base (ours)} &0.53 &0.21 &0.54 (0.01) &0.19 (0.02) &0.53 (0.00) &0.19 (0.02) &0.52 (0.01) &0.17 (0.04) & \textbf{\textcolor{blue}{0.10}}\\
     \hline
     \textbf{VFLGAN (ours)} &0.53 &0.21 &0.52 (0.01) &0.16 (0.05) &0.51 (0.02) &0.19 (0.02) &0.51 (0.02) &0.17 (0.04) & 0.16\\
     \hline

    \multicolumn{10}{c}{\textbf{Adult Dataset $(10, 1\times 10^{10-5})$-DP}} \\
    \hline
    {\textbf{WGAN\_GP}\cite{gulrajani2017improved}} &0.82 &0.72 &0.84 (0.02) &0.78 (0.06) &0.82 (0.00) &0.72 (0.00) &0.75 (0.07) &0.71 (0.01) & \textbf{\textcolor{blue}{0.16}} \\
     \hline
     \textbf{VertiGAN}\cite{Xue01} &0.82 &0.72 &0.66 (0.16) &0.40 (0.32) &0.56 (0.26) &0.50 (0.22) &0.75 (0.07) &0.43 (0.29) & 1.32\\  
     \hline
    \textbf{DPLT}  &0.82 &0.72 &0.60 (0.22) &0.55 (0.17) &0.48 (0.34) &0.44 (0.28) &0.31 (0.51) &0.30 (0.42) & {1.94} \\ 
    \hline
    \textbf{VFLGAN-base (ours)} &0.82 &0.72 &0.82 (0.00)&0.72 (0.00) &0.80 (0.02) &0.69 (0.03)&0.82 (0.00) &0.72 (0.00) & \textbf{\textcolor{red}{0.05}}\\
     \hline
      \textbf{VFLGAN (ours)} &0.82 &0.72 &0.83 (0.01) &0.78 (0.06) &0.82 (0.00) &0.73 (0.01) &0.75 (0.07) &0.71 (0.01) & \textbf{\textcolor{blue}{0.16}}\\  
     \hline
       \multicolumn{10}{c}{\textbf{Credit Dataset $(10, 1\times 10^{10-3})$-DP}} \\
    \hline
    {\textbf{WGAN\_GP}\cite{gulrajani2017improved}} &0.75 &0.67  &0.78 (0.03) &0.70 (0.03) &0.74 (0.01) &0.63 (0.04) &0.71 (0.04) &0.63 (0.04) & \textbf{\textcolor{red}{0.19}} \\
     \hline
     \textbf{VertiGAN}\cite{Xue01}&0.75 &0.67  &0.70 (0.05) &0.44 (0.23) &0.64 (0.11) &0.48 (0.19) &0.69 (0.06) &0.42 (0.25) & 0.89\\ 
     \hline
    \textbf{DPLT}  &0.75 &0.67 &0.54 (0.21) &0.45 (0.22) &0.53 (0.22) &0.48 (0.19) &0.65 (0.10) &0.53 (0.14) & {1.08} \\ 
    \hline
    \textbf{VFLGAN-base (ours)} &0.75 &0.67  &0.66 (0.09)&0.43 (0.24) &0.63 (0.14) &0.49 (0.18)&0.69 (0.06) &0.45 (0.22) & {0.93}\\
     \hline
      \textbf{VFLGAN (ours)} &0.75 &0.67 &0.72 (0.03) &0.65 (0.02) &0.67 (0.08) &0.54 (0.13) &0.67 (0.08) &0.56 (0.11) & \textbf{\textcolor{blue}{0.45}}\\  
     \hline
     \multicolumn{10}{c}{\textbf{HCVEGY Dataset $(10, 8\times 10^{10-4})$-DP}} \\
    \hline
    {\textbf{WGAN\_GP}\cite{gulrajani2017improved}} &0.23 &0.20 &0.28 (0.05) &0.22 (0.02) &0.26 (0.03) &0.23 (0.03) &0.24 (0.01) &0.12 (0.08) & \textbf{\textcolor{blue}{0.22}} \\
     \hline
     \textbf{VertiGAN}\cite{Xue01}  &0.23 &0.20  &0.29 (0.06)&0.19 (0.01) &0.25 (0.02) &0.22 (0.02)&0.26 (0.03) &0.12 (0.08) & \textbf{\textcolor{blue}{0.22}}\\
    \hline
    \textbf{DPLT} &0.23 &0.20  &0.41 (0.18) &0.17 (0.03) &0.30 (0.07) &0.25 (0.05) &0.25 (0.02) &0.11 (0.09) & {0.44} \\ 
     \hline
    \textbf{VFLGAN-base (ours)} &0.23 &0.20 &0.31 (0.08)&0.25 (0.05) &0.24 (0.01) &0.21 (0.01)&0.24 (0.01) &0.10 (0.10) & {0.26}\\
     \hline
      \textbf{VFLGAN (ours)}  &0.23 &0.20 &0.31 (0.08) &0.20 (0.00) &0.23 (0.00) &0.21 (0.01) &0.25 (0.02) &0.17 (0.03) & \textbf{\textcolor{red}{0.14}} \\ 
     \hline
     
    \end{tabular}}
    \end{center}
    }
    \begin{tablenotes}
        \footnotesize
        \item\textbf{TRTR}: Train on real test on real; \textbf{TSTS}: train on synthetic test on synthetic; \textbf{TRTS}: train on real test on synthetic; \textbf{TSTR}: train on synthetic test on real;
        \textbf{Ac}: accuracy; \textbf{F1}: F1-score.
    \end{tablenotes}
\end{table*}

\subsubsection{Tabular Datasets}
Now, we evaluate the proposed methods on five tabular datasets, i.e., Adult \cite{misc_adult_2}, Red-Wine-Quality and White-Wine-Quality \cite{misc_wine_quality_186}, Credit \cite{misc_south_german_credit_522}, and HCV \cite{misc_hepatitis_c_virus_for_egyptian_patients_503}.
The upper five figures in Fig. \ref{fig:fid_all_datasets} show the FD curves of various methods during training on the five datasets, and the lowest FD scores during training are highlighted in the figures. From Fig. \ref{fig:fid_all_datasets}, we can see that VFLGAN and VFLGAN-base show similar performance to WGAN\_GP (trained in a centric manner) on all five datasets. Besides, VFLGAN and VFLGAN-base show superior performance than VertiGAN on all five datasets w.r.t. the lowest FD score, which means the performance of VFLGAN and VFLGAN-base almost reach the upper bound. Notably, the performance of VertiGAN w.r.t. FD score is far worse than other methods on Adult, Red-Wine-Quality, and White-Wine-Quality datasets while the performance of VertiGAN is close to other methods on Credit and HCV datasets. One possible reason is that some attributes have a large value range in Adult, Red-Wine-Quality, and White-Wine-Quality datasets and VertiGAN is not effective in learning the distribution of those attributes.
On the contrary, the attributes' value ranges are limited in Credit and HCV datasets and VertiGAN achieves better performance w.r.t. FD. 

We choose the models with the lowest FD among all training epochs to generate synthetic datasets with the same size as the real dataset. Then, we train random forest (RF) models to classify the data samples by the classification target attributes in the five datasets, using all other attributes under TRTR, TSTS, TRTS, and TSTR settings as introduced in Section \ref{sec:utility metrics}.
We apply the accuracy and F1 score to evaluate the performance of RF models. Note that under TRTR and TSTS settings, we apply cross-validation with 10 evenly split sub-sets and report the mean accuracy and F1 score. The experiment results of non-DP methods are shown in the upper five sub-tables in Table \ref{tab:utility classifiction}. TRTR setting provides the baseline accuracy and F1 score.
The accuracy and F1 scores of other settings should be similar to those under the TRTR setting if the synthetic dataset is similar to the real dataset.
Thus, we calculate the absolute difference of the accuracy and F1 scores between the TRTR and other settings as the evaluation metric.
From the upper five sub-tables in Table \ref{tab:utility classifiction}, we can see that VFLGAN and VFLGAN-base generate better synthetic data than VertiGAN by a significant margin except for the HCV dataset. 
Moreover, the metrics of VFLGAN, VFLGAN-base, and WGAN\_GP are very close, which means the performance of VFLGAN and VFLGAN-base almost reaches the upper bound. 
For the HCV dataset, the performance of the RF model is poor even for the TRTR setting, which means that this dataset is not suited for classification. In this case, we can refer to the FD for evaluation. As we mentioned before, VFLGAN, VFLGAN-base, and WGAN\_GP outperform VertiGAN according to FD (Fig. \ref{fig:fid_all_datasets}).

\textbf{Takeaway}: The shared discriminator $D_s$ plays a pivotal role in guiding the generators towards accurately learning the correlations among attributes distributed across different parties. Additionally, the second part of the discriminators $D_1$ and $D_2$ contributes to aligning the distribution of the synthetic data $\tilde{\boldsymbol{x}}_1$ and $\tilde{\boldsymbol{x}}_2$ more closely with that of real data $\boldsymbol{x}_1$ and $\boldsymbol{x}_2$, respectively. The performance gap between VFLGAN and VFLGAN-base is noticeable when the attribute number is large and correlations are closely related, e.g., MNIST, but the performance gap between VFLGAN and VFLGAN-base is ignorable when the number of attributes is small. Furthermore, our evaluation highlights the limitations of relying solely on TSTR metrics for assessing statistical utility, as evidenced by VertiGAN's comparable performance under TSTR on the Adult dataset.
In some cases like HCV, some datasets are not suitable for classification, and we need to refer to FD for evaluation.
This finding underlines the critical importance of utilizing a diverse set of evaluation methods, including TSTS, TRTS, and FD, to ensure a comprehensive analysis of synthetic datasets. Moreover, we observe that while $k$-way AVD applied in \cite{8667703, Xue01} can measure the similarity of discrete attributes, it becomes computationally intensive as $k$ increases and cannot support continuous data. Conversely, FD provides an efficient means to quantify the overall attribute similarity with minor computational demand, facilitating the selection of the optimal model with the lowest FD across training epochs. 
These insights advocate for the adoption of these metrics as standards in future research.

\subsection{Utility Results of DP-VFLGAN}

In this section, we first evaluate the effectiveness of the proposed Gaussian mechanism. Then, we evaluate the utility of the proposed DP-VFLGAN. 

\subsubsection{Effectiveness of the Proposed Gaussian Mechanism} \label{sec:effectiveness of DP}

We use the MNIST dataset to evaluate the effectiveness of different mechanisms. The DP budget is set as $(10, 1\times 10^{-5})$, i.e., the generators trained by all mechanisms satisfy $(10, 1\times 10^{-5})$-DP. 
As mentioned before, we apply the same network architecture as the official implementation of WGAN\_GP when evaluating various differential privacy methods. The detailed architecture of WGAN\_GP is presented in \textbf{Appendix} \ref{sec:Architectures}.
After training, we choose the generators with the lowest FD among training epochs to generate synthetic samples and the results are shown in Fig. \ref{fig:MNIST_dp_demo}. From the figure, we can see that our proposed Gaussian mechanism provides significantly better digits compared with that of DPSGD-GAN and GS\_WGAN when satisfying the same DP guarantee. There are only two differences between our implementation of GS\_GAN and that of the original paper. First, the model architectures are different for fair comparison. Second, we use unconditional GAN and conditional GAN is applied in the original paper. 
Considering the excellent results shown in \cite{chen2020gs}, We think GS\_WGAN relies on the advanced network architecture applied in the official implementation. 

\textbf{Takeaway}:
Compared with DPSGD, our proposed Gaussian mechanism introduces significantly less noise to the parameter gradients, thus preserving more useful information. Unlike GS\_WGAN, our mechanism does not necessitate partitioning the training dataset into multiple subsets. It is widely recognized that smaller datasets can adversely affect the performance of GANs. Consequently, when operating under the same privacy budget, the GAN trained with our Gaussian mechanism consistently outperforms its counterparts. 

\begin{figure}[htbp]
\center{\includegraphics[width=0.7\linewidth, scale=1.]{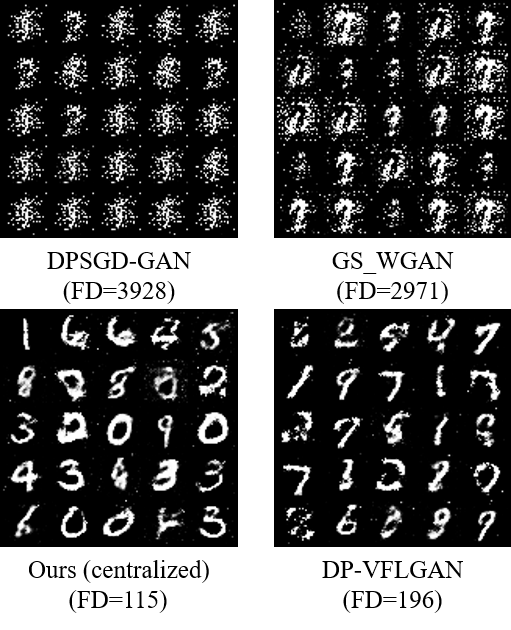}}
\caption{Synthetic digit samples satisfying $(10, 1\times 10^{-5})$-DP.}
\label{fig:MNIST_dp_demo}
\vspace{-2mm} 
\end{figure}

\subsubsection{Utility Results of DP-VFLGAN}
We apply the proposed Gaussian mechanism to WGAN\_GP to provide it with a DP guarantee. We also apply the proposed Gaussian mechanism to VertiGAN for fair comparison instead of using the DPSGD, which is applied in the original paper, since we have shown that our mechanism can provide a much tighter DP guarantee than DPSGD. We set $\epsilon=10$ for all methods. For Adult and MNIST datasets, we set $\delta=1\times 10^{-5}$ like \cite{chen2020gs}. We set $\delta=5\times 10^{-4}$ and $\delta=2\times 10^{-4}$ for Red-Wine-Quality and White-Wine-Quality, respectively, since they are small datasets with only 1599 and 4898 records, respectively. We set $\delta=8\times 10^{-4}$ and $\delta=1\times 10^{-3}$ for HCV and Credit, respectively, since they are small datasets with only 1385 and 1000 records, respectively. We use the same $T_{max}$ and $\sigma$ for all the DL methods, which are computed according to Section \ref{sec:RDP Guarantee} to meet the DP budget.

The performance of the proposed DP-VFLGAN on the MNIST dataset is shown in Fig. \ref{fig:MNIST_dp_demo}. We can see the FD of DP-VFLGAN is lower than the FDs of non-DP VFLGAN-base and non-DP VertiGAN in Fig \ref{fig:MNIST_demo}, which shows the effectiveness of both VFLGAN and the proposed Gaussian mechanism. The FD curves during training are shown in the bottom five figures of Fig. \ref{fig:fid_all_datasets}. Similar to the non-DP methods, the DP-WGAN\_GP, DP-VFLGAN, and DP-VFLGAN-base achieve lower FD than DP-VertiGAN. We do not show the FD performance of DPLT in Fig. \ref{fig:fid_all_datasets} since it is too high compared with those of GAN-based methods. 
The AI training utility of different DP methods is shown in Table \ref{tab:utility classifiction}. 
We can see that compared with the non-DP versions of WGAN\_GP, VFLGAN, and VFLGAN-base, the performance drop of the corresponding DP versions is limited, which shows the effectiveness of the proposed Gaussian mechanism. Moreover, the AI training utility of DP-VFLGAN and DP-VFLGAN-base is close to that of DP-WGAN\_GP, which shows the effectiveness of the proposed VFLGAN framework. The performance increase of DP-VFLGAN-base and DP-VFLGAN on the Adult dataset may be caused by randomness during training. In addition, all GAN-based methods show better performance than DPLT w.r.t. the AI training utility.

\subsection{Empirical Privacy Analysis}

In this section, we first audit the privacy risk of the synthetic datasets generated by VFLGAN and DP-VFLGAN with ASSD. Then, we show that the proposed ASSD is robust compared to current MI attacks for synthetic datasets. Last, we estimate the privacy risk of intermediate features of ASIF.

\begin{table}[htbp]
\centering
{
\renewcommand\arraystretch{.8}
    \caption{Mean and Standard Deviation of the MI attack Accuracy through Synthetic Datasets Generated by VFLGAN and DP-VFLGAN.}
    \label{tab:privacy leakage}
    \begin{tabular}{c|cc|cc}
    \hline
     {} & \textbf{Naive}  &\textbf{Corr} & \textbf{Naive} &  \textbf{Corr} \\
    \hline
     {} & \multicolumn{2}{c|}{\textbf{Record 1 (Adult)}} & \multicolumn{2}{c}{\textbf{Record 2 (Adult)}} \\
    \hline
    {Non-DP} & 0.65(0.09) &  0.59(0.05) & 0.58(0.10) & 0.55(0.11) \\
    \hline
    {$(10, 10^{-5})$-DP} & 0.46(0.06) &  0.47(0.13) & 0.51(0.12) & 0.44(0.10) \\
    \hline
    {$(5, 10^{-5})$-DP} & 0.50(0.09) &  0.49(0.10) & 0.50(0.08) & 0.48(0.11) \\
    \hline
    {} & \multicolumn{2}{c|}{\textbf{Record 1 (R wine)}} & \multicolumn{2}{c}{\textbf{Record 2 (R wine)}} \\
    \hline
    {Non-DP} & 0.59(0.11) &  0.48(0.07) & 0.51(0.05) & 0.51(0.11) \\
    \hline
    \scriptsize{$(10, 5\times 10^{10-4})$-DP} & 0.53(0.12) &  0.47(0.13) & 0.52(0.18) & 0.48(0.12) \\
    \hline
    \scriptsize{$(5, 5\times 10^{10-4})$-DP} & 0.48(0.11) &  0.50(0.07) & 0.43(0.09) & 0.51(0.12) \\
    \hline
    {} & \multicolumn{2}{c|}{\textbf{Record 1 (W wine)}} & \multicolumn{2}{c}{\textbf{Record 2 (W wine)}} \\
    \hline
    {Non-DP} & 0.55(0.12) &  0.57(0.12) & 0.60(0.06) & 0.50(0.09) \\
    \hline
    \scriptsize{$(10, 2\times 10^{10-4})$-DP} & 0.54(0.11) &  0.39(0.15) & 0.52(0.05) & 0.51(0.06) \\
    \hline
    \scriptsize{$(5, 2\times 10^{10-4})$-DP} & 0.44(0.03) &  0.51(0.14) & 0.53(0.09) & 0.51(0.13) \\
    \hline
    \end{tabular}
    }
    \begin{tablenotes}
        \footnotesize
        \item\textbf{Naive}: The Adversary using naive feature; \textbf{Corr}: The Adversary using correlation feature; Non-DP: VFLGAN; $(\epsilon, \delta)$-DP: DP-VFLGAN providing $(\epsilon, \delta)$-DP guarantee.
    \end{tablenotes}
    \vspace{-4mm}
\end{table}

\begin{table*}[htbp]
\centering
{
\renewcommand\arraystretch{0.8}
    \caption{AUC Performance of Different Attacking Methods on Synthetic Datasets.}
    \label{tab:attacks on synthetic datasets}
\begin{center}
    \begin{tabular}{cccccccc}
    \hline
    \footnotesize{\diagbox{Dataset}{Methods}} & LOGAN 0 \cite{hayes2017logan} & LOGAN D1 \cite{hayes2017logan} & MC \cite{hilprecht2019monte} & GAN-leaks 0 \cite{chen2020gan} & GAN-leaks CAL \cite{chen2020gan} & DOMIAS \cite{van2023membership}& ASSD(ours) \\
    \hline
    \footnotesize{R-Wine (100/1000)} & 0.47 & 0.48 & 0.60 & 0.61 & 0.58 & 0.63 & \textcolor{red}{\textbf{0.98}} \\
    \hline
    \footnotesize{R-Wine (200/1000)} & 0.52 & 0.50 & 0.52 & 0.55 & 0.55 & 0.54 & \textcolor{red}{\textbf{0.95}} \\
    \hline
    \footnotesize{W-Wine (100/2000)} & 0.48 & 0.46 & 0.53 & 0.54 & 0.57 & 0.60 & \textcolor{red}{\textbf{0.87}} \\
    \hline
    \footnotesize{W-Wine (200/2000)} & 0.50 & 0.50 & 0.55 & 0.56 & 0.59 & 0.58 & \textcolor{red}{\textbf{0.92}} \\
    \hline
    \footnotesize{Housing (100/10000)} & 0.51 & 0.52 & 0.54 & 0.57 & 0.55 & 0.60 & \textcolor{red}{\textbf{0.81}} \\
    \hline
    \footnotesize{Housing (200/10000)} & 0.54 & 0.54 & 0.50 & 0.51 & 0.50 & 0.50 & \textcolor{red}{\textbf{0.89}} \\
    \hline
    \end{tabular}
\end{center}
    }
    \begin{tablenotes}
        \footnotesize
        \item{Dataset (\#$N_1$/\#$N_2$)}: the size of training dataset and test dataset is $N_1$ and the size of reference dataset is $N_2$. Note the proposed auditing method (ours) does not need test and reference datasets. 
    \end{tablenotes}
    \vspace{-4mm}
\end{table*}

\subsubsection{Auditing Synthetic Dataset with ASSD} \label{sec:implementation of ASSD}
In implementing the MI Attack, our approach involved two distinct phases of generator training. First, we trained 100 generators on the entire private dataset, each with a unique random seed. Concurrently, we trained another set of 100 generators on the same dataset but excluding a specific target record, ensuring each generator had a different random seed. This setup enabled us to generate 200 synthetic datasets tailored to assess the presence of the target record. For the training of our model, we randomly selected 140 synthetic datasets—70 generated from the complete dataset and 70 excluding the target record. The remaining 60 datasets were used for testing purposes. The core of our analysis involved training a Random Forest model to identify whether the target record was part of the training data, based on the characteristics of these synthetic datasets. After testing the Random Forest model, we repeated the selection, training, and testing process five times to ensure robustness in our results. This repetition allowed us to calculate the mean and standard deviation of the attack's AUC score, providing a quantitative measure of privacy leakage. A higher AUC score indicated a more significant privacy risk, while an AUC score at or below 0.5 suggested minimal privacy exposure. Consistent with methodologies in previous studies \cite{Theresa01}, we employed naive and correlation features for the MI attacks. We decided not to include results from histogram features in our analysis, as they showed little potential for privacy compromise in our experiments. 

We apply the methods of \cite{Theresa01} and \cite{meeus2023achilles}, to select the most vulnerable records in the training data with the proposed MI attack. Detailed methodologies for selecting these vulnerable records are described in \textbf{Appendix} \ref{sec:vulnerability}. We chose the two most vulnerable records from each dataset for this analysis to measure the potential privacy leakage.
The results of this assessment are presented in Table \ref{tab:privacy leakage}. We observe that the attack AUC scores on datasets generated by VFLGAN consistently exceed 50\%, indicating a certain risk of privacy leakage. In contrast, datasets generated by DP-VFLGAN show attack AUC scores close to or below 50\%. This suggests a substantial reduction in privacy risk, highlighting the effectiveness of the DP-VFLGAN in protecting against MI attacks.

\subsubsection{Evaluating the Effectiveness of ASSD} Now, we compare the proposed ASSD with several current MI attacks for synthetic datasets proposed in \cite{hayes2017logan, hilprecht2019monte, chen2020gan, van2023membership} on three datasets, i.e., Red-Wine-Quality, White-Wine-Quality, and Housing \cite{pace1997sparse}. Housing is applied in the benchmark proposed in \cite{van2023membership}. Since those MI attacks cannot support categorical attributes, we add two new datasets to the benchmark with minimum categorical attributes, i.e., Red-Wine-Quality and White-Wine-Quality. We also apply We set the size of training datasets to 100 and 200. All the attacks in \cite{hayes2017logan, hilprecht2019monte, chen2020gan, van2023membership} require a test dataset with the same size as the training dataset. Moreover, Domias, Logan D1, and Gan-leak cal require a reference dataset, which presents the same distribution as the training dataset. On the contrary, the proposed ASSD does not require a test or reference dataset. The attack AUC scores are shown in Table \ref{tab:attacks on synthetic datasets}. ASSD present a superior performance than other attacks, which shows the effectiveness of the proposed auditing scheme. We can infer that current attacks in \cite{chen2020gan,hayes2017logan,hilprecht2019monte,van2023membership}  won't succeed when the AUC score of ASSD is small (e.g., less than $\sim$0.6).

\textbf{Takeaway}:
ASSD is more effective than other attacks because ASSD targets the vulnerable record in the training dataset selected by the method of \cite{Theresa01} while other attacks target every record in the training dataset. Besides, ASSD applies the leave-one-out setting, which is stronger than auxiliary datasets but more realistic and practical for data holders. The notable disparity in AUC scores of ASSD between Table \ref{tab:attacks on synthetic datasets} and Table \ref{tab:privacy leakage} can be attributed to variations in the sizes of the respective training datasets. In Table \ref{tab:attacks on synthetic datasets}, the size of training datasets is 100 and 200 while in Table \ref{tab:privacy leakage}, the whole datasets are used to train the models. It's well known that a smaller size can cause more privacy leakage \cite{van2023membership}.

\subsubsection{Auditing Intermediate Features with ASIF}
The implementation of ASIF is similar to that of ASSD in Section \ref{sec:implementation of ASSD}. We also train 200 pairs ($D_1$ and $D_2$) of local discriminators, in which 100 pairs are trained on the complete dataset and 100 pairs are trained on the same dataset but excluding a specific target record ($x_t$), which is selected by the methods in \cite{Theresa01}. We input the complete dataset to the local discriminators to generate 200 sets of intermediate features, of which 140 sets are used for training, and 60 sets are used for the test. Then, we train a random forest model to detect whether $x_t$ is in the training dataset given a set of intermediate features. From Table \ref{tab:IF privacy leakage}, we can see that the privacy risk of intermediate features is minor, even for non-DP models.

\begin{table}[htbp]
\centering
{
\renewcommand\arraystretch{.8}
    \caption{Mean and Standard Deviation of the MI attack Accuracy through Intermediate Features Generated by VFLGAN and DP-VFLGAN.}
    \label{tab:IF privacy leakage}
    \begin{tabular}{c|cc|cc}
    \hline
     {} & \textbf{Naive}  &\textbf{Corr} & \textbf{Naive} &  \textbf{Corr} \\
    \hline
    
     {} & \multicolumn{2}{c|}{\textbf{Record 1 (Adult)}} & \multicolumn{2}{c}{\textbf{Record 2 (Adult)}} \\
    \hline

    {Non-DP} & {0.49(0.12)} &  {0.58(0.09)} & {0.53(0.09)} & {0.55(0.12)} \\
    \hline

    {$(10, 10^{-5})$-DP} & {0.49(0.14)} &  {0.50(0.08)} & {0.54(0.11)} & {0.48(0.12)} \\
    \hline

    {$(5, 10^{-5})$-DP} &{0.49(0.03)} & {0.50(0.12)} & {0.54(0.09)} & {0.51(0.10)} \\
    \hline

    {} & \multicolumn{2}{c|}{\textbf{Record 1 (R wine)}} & \multicolumn{2}{c}{\textbf{Record 2 (R wine)}} \\
    \hline

    {Non-DP} & 0.54(0.09) &  0.55(0.14) & 0.52(0.17) & 0.54(0.12) \\
    \hline

    \scriptsize{$(10, 5\times 10^{10-4})$-DP} & 0.53(0.04) &  0.50(0.07) & 0.54(0.09) & 0.50(0.09) \\
    \hline

    \scriptsize{$(5, 5\times 10^{10-4})$-DP} & 0.51(0.09) &  0.45(0.03) & 0.54(0.02) & 0.53(0.07) \\
    \hline

    {} & \multicolumn{2}{c|}{\textbf{Record 1 (W wine)}} & \multicolumn{2}{c}{\textbf{Record 2 (W wine)}} \\
    \hline

    {Non-DP} &{0.52(0.10)} & {0.50(0.08)} &{0.51(0.09)} &{0.49(0.13)} \\
    \hline

    \scriptsize{$(10, 2\times 10^{10-4})$-DP} &{0.52(0.10)} &  {0.52(0.12)} & {0.48(0.10)} & {0.51(0.12)} \\
    \hline

    \scriptsize{$(5, 2\times 10^{10-4})$-DP} & {0.49(0.14)} & {0.56(0.07)} &{0.52(0.07)} &{0.53(0.08)} \\
    \hline
    
    \end{tabular}
    }
    \begin{tablenotes}
        \footnotesize
        \item\textbf{Naive}: The Adversary using naive feature; \textbf{Corr}: The Adversary using correlation feature; Non-DP: VFLGAN; $(\epsilon, \delta)$-DP: DP-VFLGAN providing $(\epsilon, \delta)$-DP guarantee.
    \end{tablenotes}
    \vspace{-6mm} 
\end{table}

\section{Conclusion and Future Work}

In this paper, we found that VertiGAN published in PETS 2023 \cite{Xue01} can not effectively learn the correlation among attributes between different parties, which leads to the distribution of synthetic data being different from that of the real data. To resolve this issue, we proposed the \emph{first} VFL-based GAN, VFLGAN, for vertically partitioned data publication. Experiment results show that VFLGAN significantly improves the quality of synthetic data compared with the state-of-the-art methods. We also proposed a Gaussian mechanism for VFLGAN to make the generators satisfy a rigorous DP guarantee. Experimental results show that the proposed Gaussian mechanism can produce better synthetic data compared with current differentially private mechanisms under the same DP budget. Besides, the utility drop of DP-VFLGAN is limited compared to its non-DP version. Additionally, we propose a practical privacy leakage measurement with realistic assumptions since the DP is too conservative. Our experimental results show that DP-VFLGAN can effectively mitigate privacy breaches. However, the proposed privacy leakage measurement can only estimate the privacy breach from the view of external attackers. Our future work will be on estimating the privacy breach from the view of internal attackers.

\clearpage
\clearpage

\bibliographystyle{ACM-Reference-Format}
\bibliography{sample-base}


\appendix

\section{Appendix}

\subsection{Literature Review} \label{sec:literature review}

This section presents a detailed literature review. In this regard, we first introduce the recent works about vertically partitioned data publication methods that provide DP guarantees. Next, we review multiple differentially private mechanisms for GANs and explain why they are unsuitable for the proposed VFLGAN. Last, we describe some privacy measurements and attacks on synthetic datasets inspired by which we propose our practical auditing schemes to measure privacy leakage. 
\subsubsection{{Vertically Partitioned Data Publication}}
In \cite{6517175}, the authors proposed a secure two-party algorithm, DistDiffGen, that applies the exponential mechanism for vertically partitioned data publication and satisfies $\epsilon$-DP. However, their protocol only works under the two-party scenario and the data utility deteriorates fast when the number of attributes increases \cite{6517175}. Besides, DistDiffGen is tailored to classification tasks and cannot guarantee meaningful utility for many common data analysis tasks \cite{8667703}.
In \cite{8667703}, the authors proposed DPLT for vertically partitioned data publication, which satisfies $\epsilon$-DP. However, DPLT is limited to discrete datasets. Besides, DPLT evenly splits the privacy budget over all the attribute pairs, which is unreasonable since the information leakage levels of different attributes are usually different. As pointed out in \cite{Xue01}, the noise scale may increase exponentially with the increased data dimensionality, which can cause a significant utility loss. The first GAN-based vertically partitioned data publication method, VertiGAN, was proposed in \cite{Xue01} to solve the above issues. VertiGAN satisfies $(\epsilon,\delta)$-DP, where the DP guarantee is achieved by adding noise when updating discriminators' parameters so the privacy budget can be distributed more intelligently among attributes and the curse of dimensionality can be mitigated. However, the discriminator is updated according to FedAvg \cite{mcmahan2017communication}, a Horizontal Federated Learning (HFL)-based method. However, VeriGAN is less effective in learning the correlation among the attributes of different parties. This paper applies a VFL framework to learn the correlation mentioned above.
\subsubsection{{Differentially Private Generative Adversary Networks}} \label{sec:literature review DP}

For general deep learning models, there are mainly two kinds of training methods to make the trained models satisfy differential privacy, i.e., Differentially Private Stochastic Gradient Descent (DPSGD) \cite{shokri2015privacy, yu2019differentially} and Private Aggregation of Teacher Ensembles (PATE) \cite{papernot2018scalable}. In \cite{xie2018differentially, zhang2018differentially}, the authors proposed variants of DPSGD to train the discriminator and apply non-private SGD to train the generator. The rationale is that only the discriminator can access the training data, and the generator learns from the discriminator to generate better synthetic data. According to the post-processing theorem, Proposition \ref{theorem: post-processing}, the generator satisfies the same level of DP as the discriminator \cite{xie2018differentially, zhang2018differentially}. However, this method is unsuitable for our VFLGAN, shown in Fig. \ref{fig:vflmodel}, since each party cannot control the parameter update of the shared discriminator. PATE-GAN \cite{jordon2018pate} adopts the PATE framework to provide differential privacy. First, multiple non-private discriminators are trained on non-overlapping subsets. Then, the non-private discriminators are applied to train a student discriminator that satisfies $(\epsilon, \delta)$-DP. Last, the student discriminator is used to train the generator. However, this method is inapplicable to our VFL-based scheme since the non-private teacher discriminators can cause information leakage through the shared discriminator. GS-WGAN \cite{chen2020gs} provides another solution. The discriminators are trained in a non-private way while the backward gradients between the discriminators and the generator are sanitized with a Gaussian mechanism to make the generator satisfy a $(\epsilon,\delta)$-DP. Besides, the privacy guarantee is enhanced by the subsampling procedure \cite{wang2019subsampled}, i.e., dividing the training dataset into multiple non-overlapping subsets as \cite{jordon2018pate} and training multiple discriminators on each subset. However, this framework cannot be adapted to VFLGAN since the information can be leaked through the shared discriminator when training non-private discriminators.

\subsubsection{{Privacy Measurements and Attacks on Synthetic Dataset}}
Some papers \cite{yale2019assessing, lu2019empirical, giomi2023unified} utilize the distribution similarity between the synthetic dataset and the training dataset to measure the potential privacy leakage.
In \cite{yale2019assessing}, the privacy loss is measured by nearest neighbour adversarial accuracy, and both training and test datasets are required to calculate the privacy loss. 
In \cite{lu2019empirical}, the authors first sort the original and synthetic data records according to a predefined metric, then compare the number of matches according to the rank of ordered records where more matches indicate a higher risk of re-identification. 
In \cite{giomi2023unified}, the authors proposed singling out attack, linkability attack, and inference attack against synthetic datasets, which assume there is a control dataset following the same distribution as the training dataset and measured the privacy leakage by the attack accuracy. However, the above works do not follow the DP principle when measuring privacy leakage. Instead, they mainly measure the overfitting level of the GAN to the training data as privacy leakage. 

The principle of DP was first adopted in \cite{Theresa01} to design privacy attacks against synthetic datasets. The authors of \cite{Theresa01} apply shadow models \cite{Reza01} to launch an MI attack that tries to distinguish whether the training dataset contains a specific record given the generated synthetic dataset. Shadow models MI attack is also applied in \cite{houssiau01}, which is further enhanced by an advanced feature map. However, these shadow model attacks require a large amount of reference data which follows the same distribution as the training dataset. For example, the size of the reference dataset is ten times larger than the training dataset in \cite{Theresa01}. The feasibility of this assumption in the context of privacy auditing is questionable, as expecting the data publisher to employ an excessively large dataset for auditing purposes is impractical. Moreover, as pointed out in \cite{Jiayuan01}, improper reference datasets can overestimate or underestimate the privacy loss. 
In \cite{guepin2023synthetic}, the authors proposed an MI attack requiring only the targeted synthetic dataset. However, the double-counting phenomenon has not been resolved.
In this paper, we adopt the leave-one-out setting proposed in \cite{Jiayuan01} and shadow models to measure the privacy leakage of any specific record, which satisfies the principle of DP and does not require a reference dataset.

There is another research line about MI attacks on synthetic datasets \cite{hayes2017logan, hilprecht2019monte, chen2020gan, van2023membership}. These attacks have the same black-box assumption as our auditing method, i.e., the GAN models are not accessible to attackers. However, the attacks in \cite{hayes2017logan, hilprecht2019monte, chen2020gan, van2023membership} require auxiliary datasets and targets every record in the training dataset. On the contrary, our auditing method does not require any auxiliary dataset so that data holders can use all of their data to train the generative model. Besides, our auditing method targets one vulnerable data record and thus achieves superior performance than the attacks in \cite{hayes2017logan, hilprecht2019monte, chen2020gan, van2023membership}. On the other hand, \cite{jagielski2020auditing, lu2022general, steinke2024privacy, krishna2023unleashing} aim to estimate the lower bounds on $\epsilon$ for $\epsilon$-DP. There are three major differences between our auditing method and those in \cite{jagielski2020auditing, lu2022general, steinke2024privacy, krishna2023unleashing}. First of all, methods in \cite{jagielski2020auditing, lu2022general, steinke2024privacy, krishna2023unleashing} estimate the worst
privacy risk of ALL possible inputs while our auditing method estimates the privacy risk of the record in training datasets (which is also the focus of data holders). Second, methods in \cite{jagielski2020auditing, lu2022general, steinke2024privacy, krishna2023unleashing} are designed for classification models and apply the classification losses to conduct attacks. However, there is no such classification loss in our VFLGAN. Third, in the vertically partitioned data publication scenario, models are assumed to be not accessible to attackers, as in \cite{hayes2017logan, hilprecht2019monte, chen2020gan, van2023membership}. However, methods in \cite{jagielski2020auditing, lu2022general, steinke2024privacy, krishna2023unleashing} require submitting requests to the classification models and receiving the results that are used to launch the attacks.

\begin{table*}[htbp]
\centering
{
\renewcommand\arraystretch{0.9}
    \caption{Classification Accuracy and F1 score of Random Forest Models to Evaluate the AI Training Utility of the Synthetic Data on the Wine-Quality Datasets. Absolute different metric values compared with the TRTR setting are shown in the bracket, and the main evaluation metric is `Total Difference'; \textcolor{red}{lower} is better.}
    \label{tab:red-wine}
    \begin{tabular}{ccccccccc|c}
    \hline
    \multicolumn{10}{c}{\textbf{Normalized Integer Representation for the Quality attribute (red wine)}} \\
    \hline
    {} & \multicolumn{2}{c}{\textbf{TRTR}} & \multicolumn{2}{c}{\textbf{TSTS}} & \multicolumn{2}{c}{\textbf{TRTS}} & \multicolumn{2}{c}{\textbf{TSTR}} & \textbf{Total Difference}\\
    \textbf{Models} & \textbf{Acc} & \textbf{F1} & \textbf{Acc} & \textbf{F1} & \textbf{Acc} & \textbf{F1} & \textbf{Acc} & \textbf{F1}  & {} \\
     \hline
     {\textbf{WGAN\_GP}\cite{gulrajani2017improved}} &0.59 &0.26 &0.62 (0.03) &0.45 (0.19) &0.53 (0.06) &0.20 (0.06) &0.54 (0.05) &0.29 (0.03) & 0.42\\     
     \hline
     \textbf{VertiGAN}\cite{Xue01} &0.59 &0.26 &0.92 (0.33) &0.67 (0.41) &0.80 (0.21) &0.41 (0.15) &0.55 (0.04) &0.26 (0.00) & 1.14\\  
     \hline
     \textbf{VFLGAN-base (ours)} &0.59 &0.26 &0.74 (0.15) &0.28 (0.02) &0.66 (0.07) &0.25 (0.01) &0.56 (0.03) &0.26 (0.00) & 0.28\\
     \hline
     \textbf{VFLGAN (ours)} &0.59 &0.26 &0.74 (0.15) &0.44 (0.18) &0.59 (0.00) &0.29 (0.03) &0.51 (0.08) &0.26 (0.00) & 0.44\\ 
    \hline

    \multicolumn{10}{c}{\textbf{Normalized Integer Representation for the Quality attribute (white wine)}} \\
    \hline
     {\textbf{WGAN\_GP}\cite{gulrajani2017improved}} &0.53 &0.21 &0.49 (0.04) &0.17 (0.04) &0.48 (0.05) &0.14 (0.07) &0.53 (0.00) &0.19 (0.02) & 0.22\\
     \hline
     \textbf{VertiGAN}\cite{Xue01} &0.53 &0.21 &0.77 (0.24) &0.48 (0.27) &0.46 (0.07) &0.15 (0.06) &0.38 (0.15) &0.14 (0.07) & 0.86\\  
    \hline
    \textbf{VFLGAN-base (ours)} &0.53 &0.21 &0.52 (0.01)&0.28 (0.07) &0.46 (0.07) &0.20 (0.01)&0.47 (0.06) &0.19 (0.02) & 0.24\\
     \hline
      \textbf{VFLGAN (ours)} &0.53 &0.21 &0.50 (0.03) &0.24 (0.03) &0.47 (0.06) &0.18 (0.03) &0.50 (0.03) &0.21 (0.00) & 0.18\\  
     \hline
    \end{tabular}
    }
    \begin{tablenotes}
        \footnotesize
        \item\textbf{TRTR}: Train on real test on real; \textbf{TSTS}: train on synthetic test on synthetic; \textbf{TRTS}: train on real test on synthetic; \textbf{TSTR}: train on synthetic test on real;
        \textbf{Ac}: accuracy; \textbf{F1}: F1-score.
    \end{tablenotes}
\end{table*}

\subsection{Proof of Theorem \ref{theorem RDP guarantee}} \label{sec:proof of RDP guarantee}
\begin{proof}
Let $\boldsymbol{x}_i^B$ and $\boldsymbol{x}_i^{B^\prime} \in X_i$ denote two adjacent mini-batches of training data of party $i$. 
The gradients of the parameters of the first layer of $D_i$ are clipped using \eqref{eq:clip}. Then the L2 norm of those gradients has the following upper bound,
\begin{equation}
\label{eq:clip in proof}
    \left\|clip(\mathcal{G}{_{D_i(\boldsymbol{x}_j^B)}^1},C)\right\|_2 \leq C, \quad i,j \in \{1,2\}.
\end{equation}
Note that although the input of $D_i$ is $\boldsymbol{x}_i^B$, the gradients of $D_i$ are affected by both $\boldsymbol{x}_1^B$ and $\boldsymbol{x}_2^B$ according to \eqref{eq:gradient_d1}. According to the triangle inequality, the L2 sensitivity of the parameters can be derived as
\begin{equation}
\label{eq:sensitivity in proof}
    \Delta_2 f=\max _{\boldsymbol{x}_j^B, \boldsymbol{x}_j^{B^\prime}}\left\|clip(\mathcal{G}{_{D_i(\boldsymbol{x}_j^B)}^1},C)-clip(\mathcal{G}{_{D_i(\boldsymbol{x}_j^{B^\prime})}^1},C)\right\|_2 \leq 2 C,
\end{equation}
where $i,j \in \{1,2\}$. According to Proposition \ref{Prop: Gaussian mechenism}, $\mathcal{G}{_{D_i}^1}$ computed by \eqref{eq:add noise} satisfies $(\alpha, \alpha/(2\sigma^2))$-RDP w.r.t. $\boldsymbol{x}_1^B$ and $\boldsymbol{x}_2^B$ since original gradients are calculated regarding $\boldsymbol{x}_1^B$ and $\boldsymbol{x}_2^B$. 
Let $\boldsymbol{x}^B=[\boldsymbol{x}_1^B,\boldsymbol{x}_2^B]$ and $\boldsymbol{x}^{B^\prime}=[\boldsymbol{x}_1^{{B^\prime}},\boldsymbol{x}_2^{{B^\prime}}]$ be two adjacent mini-batches collected from the complete dataset, i.e., $X=[X_1,X_2]$. Equations \eqref{eq:clip in proof} and \eqref{eq:sensitivity in proof} still hold if we delete the subscript of $\boldsymbol{x}_j^B$ and $\boldsymbol{x}_j^{B^\prime}$. Thus, similar to the above proving process, $\mathcal{G}{_{D_i}^1}$ computed by \eqref{eq:add noise} also satisfies $(\alpha, \alpha/(2\sigma^2))$-RDP w.r.t. $\boldsymbol{x}^B$.

According to Proposition \ref{theorem: post-processing} (Post-processing theorem), the parameters of the first layer of $D_1$ and $D_2$ updated by
\begin{equation}
    \boldsymbol{\theta}_{D_i}^1 = \boldsymbol{\theta}_{D_i}^1-\eta_{D_i}\mathcal{G}_{D_i}^1
\end{equation}
satisfy the same RDP as $\mathcal{G}_{D_i}^1$.

The outputs of the first part of $D_1$, $D_2$, i.e., the intermediate features $\boldsymbol{f}_1$ and $\boldsymbol{f}_2$, can be expressed as,
\begin{equation}
\label{eq:InF-DP}
    \boldsymbol{f}_i=func_{D_i}(\boldsymbol{\theta}_{D_i}^1 \boldsymbol{x}_i), \quad i \in \{1,2\},
\end{equation}
where $\boldsymbol{x}_i$ denotes the input of $D_i$, $func_{D_i}$ denotes the calculation after the first layer ($\boldsymbol{\theta}_{D_i}^1 \boldsymbol{x}_i$). Thus, according to Proposition \ref{theorem: post-processing}, since $\boldsymbol{\theta}^1_{D_i}$ satisfies $(\alpha, \alpha/(2\sigma^2))$-RDP, the mechanism $func_{D_i}(\boldsymbol{\theta}_{D_i}^1 \boldsymbol{x}_i)$ in \eqref{eq:InF-DP} satisfy $(\alpha, \alpha/(2\sigma^2))$-RDP, i.e., the first part of $D_1$ and $D_2$ satisfy $(\alpha, \alpha/(2\sigma^2))$-RDP.

On the other hand, according to \eqref{eq:g_loss}, the generators $G_1$ and $G_2$ can only learn the information about the input data from the gradients of the first layer of $D_1$ and $D_2$, respectively. During the back-propagation process, let $\boldsymbol{\delta}^G_{D_i}$ and $\boldsymbol{\delta}^1_{D_i}$ denote the backward gradients before and after the first layer of $D_i$, which is illustrated in Fig. \ref{fig:backprop}. Note that the parameters of the first layer of $D_i$ are updated according to $\boldsymbol{\delta}^1_{D_i}$ and the parameters of $G_i$ are updated according to $\boldsymbol{\delta}^G_{D_i}$. The relationship between $\boldsymbol{\delta}^G_{D_i}$ and $\boldsymbol{\delta}^1_{D_i}$ can be expressed as
\begin{equation}
    \boldsymbol{\delta}^G_{D_i} = \boldsymbol{\delta}^1_{D_i}\boldsymbol{\theta}_{D_i}^1.
\end{equation}
According to Proposition \ref{theorem: post-processing}, $\boldsymbol{\delta}^G_{D_i}$ satisfies the same RDP as $\boldsymbol{\theta}_{D_i}^1$.
Since the parameters of $G_i$ is updated according to $\boldsymbol{\delta}^G_{D_i}$, $G_i$ satisfies $(\alpha, \alpha/(2\sigma^2))$-RDP w.r.t. $x_1^B$, $x_2^B$, and $x^B$.
\end{proof}

\begin{figure}[htbp]
\center{\includegraphics[width=\linewidth, scale=1.2]{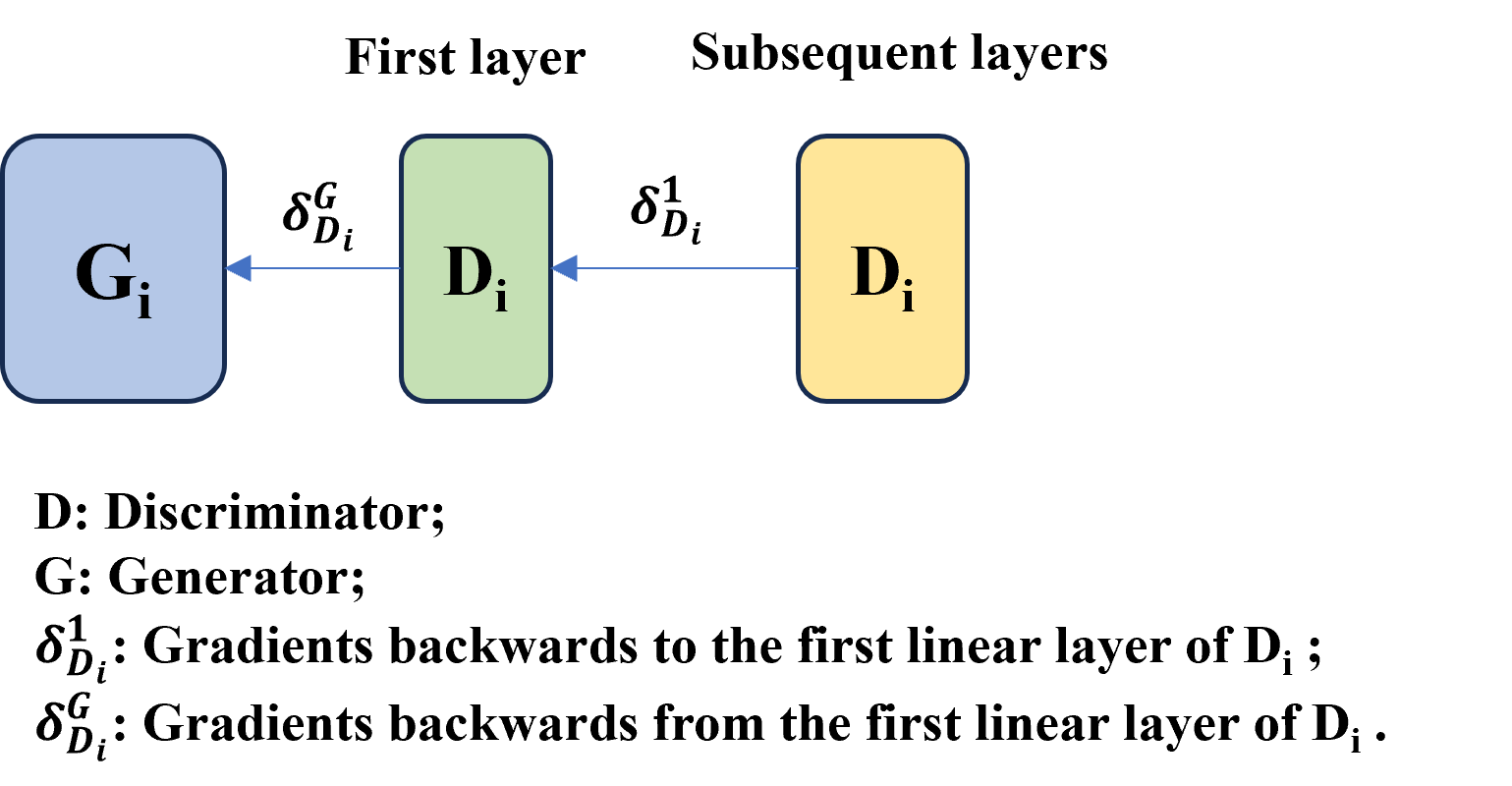}}
\caption{Backward loss before and after the first layer of the discriminators.}
\label{fig:backprop}
\end{figure}


\subsection{Auditing Scheme for Intermediate Features} \label{app:Auditing Scheme for Intermediate Features}

The following privacy game can define the MI attack applied in ASIF. \textbf{(i)} First, the challenger selects a fixed target record $\boldsymbol{x}_t$ from a given training dataset $X$. \textbf{(ii)} Then, the challenger trains a VFLGAN $\boldsymbol{\theta}_0 \stackrel{s_{0}}{\longleftarrow} \mathcal{T}(X \setminus \boldsymbol{x}_t)$ on $X \setminus \boldsymbol{x}_t$ (dataset $X$ excluding $\boldsymbol{x}_t$) by using a fresh random seed $s_{0}$ in the training algorithm $\mathcal{T}$. In this paper, $\mathcal{T}$ refers to Algorithm \ref{alg:Training DP-VFL-GAN}. \textbf{(iii)} The challenger trains another VFLGAN $\boldsymbol{\theta}_1 \stackrel{s_{1}}{\longleftarrow} \mathcal{T}(X)$ on $X$ by using a fresh random seed $s_{1}$ in algorithm $\mathcal{T}$. 
\textbf{(iv)} The challenger inputs the whole dataset $X$ into $\boldsymbol{\theta}_0(\cdot)$ and $\boldsymbol{\theta}_1(\cdot)$ to generate two sets of intermediate features $\boldsymbol{f}_0$ and $\boldsymbol{f}_1$, respectively.
\textbf{(v)} The challenger flips a random unbiased coin $b\in \{0,1\}$ and sends the set of intermediate features and target record, $\boldsymbol{f}_b$ and $\boldsymbol{x}_t$, to the adversary. \textbf{(vi)} The adversary tries to figure out the true $b$ based on the observation of $\{\boldsymbol{f}_b, \boldsymbol{x}_t\}$, i.e., $\hat{b} \leftarrow \mathcal{A}(\boldsymbol{f}_b, \boldsymbol{x}_t)$. \textbf{(vii)} Last, if $\hat{b}=b$, the adversary wins. Otherwise, the challenger wins. This MI attack is summarized in Algorithm \ref{alg:MI attack of ASIF}. 

\begin{algorithm}[ht]
\footnotesize
\caption{Membership Inference Attack}\label{alg:MI attack of ASIF}
\begin{algorithmic}[1]
\REQUIRE Training algorithm $\mathcal{T}$; dataset $X$; target record $x_t$; unbiased coin $b$; fresh random seeds $s_{0}$ and $s_{1}$; VFLGAN $\theta_i$; intermediate feature $f$.
\ENSURE Success or failure.
\noindent{\rule{0.95\linewidth}{0.4pt}}
\STATE $\boldsymbol{\theta}_0 \stackrel{s_{0}}{\longleftarrow} \mathcal{T}(X \setminus \boldsymbol{x}_t)$
\STATE $\boldsymbol{\theta}_1 \stackrel{s_{1}}{\longleftarrow} \mathcal{T}(X)$
\STATE $\boldsymbol{f}_0 \longleftarrow \boldsymbol{\theta}_0(X)$ \& $\boldsymbol{f}_1 \longleftarrow \boldsymbol{\theta}_1(X)$
\STATE $b \sim \{0,1\}$
\STATE $\hat{b} \longleftarrow \mathcal{A}(\boldsymbol{f}_b, \boldsymbol{x}_t)$
\IF{$\hat{b}==b$}
\STATE Output success
\ELSE
\STATE Output failure
\ENDIF
\end{algorithmic}
\end{algorithm}

\begin{algorithm}[ht]
\footnotesize
\caption{Training the Adversary of MI Attack} \label{alg:adversary of ASIF}
\begin{algorithmic}[1]
\REQUIRE Training algorithms $\mathcal{T}$ and $\mathcal{T}_\mathcal{A}$; dataset $X$; target record $\boldsymbol{x}_t$; random seeds $s_{0_{1:M}}$ and $s_{1_{1:M}}$; feature extraction function $Extr(\cdot)$; intermediate feature $\boldsymbol{f}$.
\ENSURE Trained $\mathcal{A}$.
\noindent{\rule{0.95\linewidth}{0.4pt}}
\STATE $\theta_{0_{1:M}} \stackrel{s_{0_{1:M}}}{\longleftarrow} \mathcal{T}(X \setminus \boldsymbol{x}_t)$
\STATE $\boldsymbol{f}_{0_{1:M}} \longleftarrow \boldsymbol{\theta}_{0_{1:M}}(X)$
\STATE $\boldsymbol{\theta}_{1_{1:M}} \stackrel{s_{1_{1:M}}}{\longleftarrow} \mathcal{T}(X)$
\STATE $\boldsymbol{f}_{1_{1:M}} \longleftarrow \boldsymbol{\theta}_{1_{1:M}}(X)$
\STATE $Feat_{0_{1:M}} \longleftarrow Extr(\boldsymbol{f}_{0_{1:M}})$ \& $Feat_{1_{1:M}} \longleftarrow Extr(\boldsymbol{f}_{1_{1:M}})$
\STATE $\mathcal{A} \longleftarrow \mathcal{T}_\mathcal{A}(Feat_{0_{1:M}}, Feat_{1_{1:M}})$
\end{algorithmic}
\end{algorithm}
Now we describe how to train the adversary model $\mathcal{A}$ in Algorithm \ref{alg:MI attack of ASIF}. We apply the shadow modelling approach \cite{Reza01} to train the adversary model $\mathcal{A}$ through the following process. \textbf{(i)} First, train $M$ VFLGANs, $\boldsymbol{\theta}_{0_{1:M}}$, on dataset $X\setminus \boldsymbol{x}_t$ by using different random seeds $s_{0_{1:M}}$ in the training algorithm $\mathcal{T}$. \textbf{(ii)} Apply $\boldsymbol{\theta}_{0_{1:M}}$ to generate $M$ intermediate features $\boldsymbol{f}_{0_{1:M}}$. \textbf{(iii)} Train $M$ VFLGANs, $\boldsymbol{\theta}_{1_{1:M}}$, on dataset $X$ by using different random seeds $s_{1_{1:M}}$ in the training algorithm $\mathcal{T}$. \textbf{(vi)} Apply $\boldsymbol{\theta}_{1_{1:M}}$ to generate $M$ sets of intermediate features $\boldsymbol{f}_{1_{1:M}}$. \textbf{(v)} Extract features of the intermediate features. \textbf{(vi)} Train the adversary model $\mathcal{A}$ to distinguish whether the features are from $\boldsymbol{f}_{0_{1:M}}$ or $\boldsymbol{f}_{1_{1:M}}$. The training process is summarized in Algorithm \ref{alg:adversary of ASIF}.

\subsection{Model Architectures} \label{sec:Architectures}

Figure \ref{fig:WGAN_architecture} shows the architecture of WGAN\_GP, which is applied when we evaluate the effectiveness of various differentially private mechanisms in Section \ref{sec:effectiveness of DP}. Our implementation of VertiGAN (no public official code) also applies the same architecture as shown in Fig. \ref{fig:WGAN_architecture}. In the proposed VFLGAN, the architecture of the generator is the same as that of the generator in Fig. \ref{fig:WGAN_architecture} while the architecture of discriminators is shown in Fig. \ref{fig:VFLGAN_architecture}. 


\begin{figure}[htbp]
\center{\includegraphics[scale=0.50,trim= 100 50 400 120,clip]{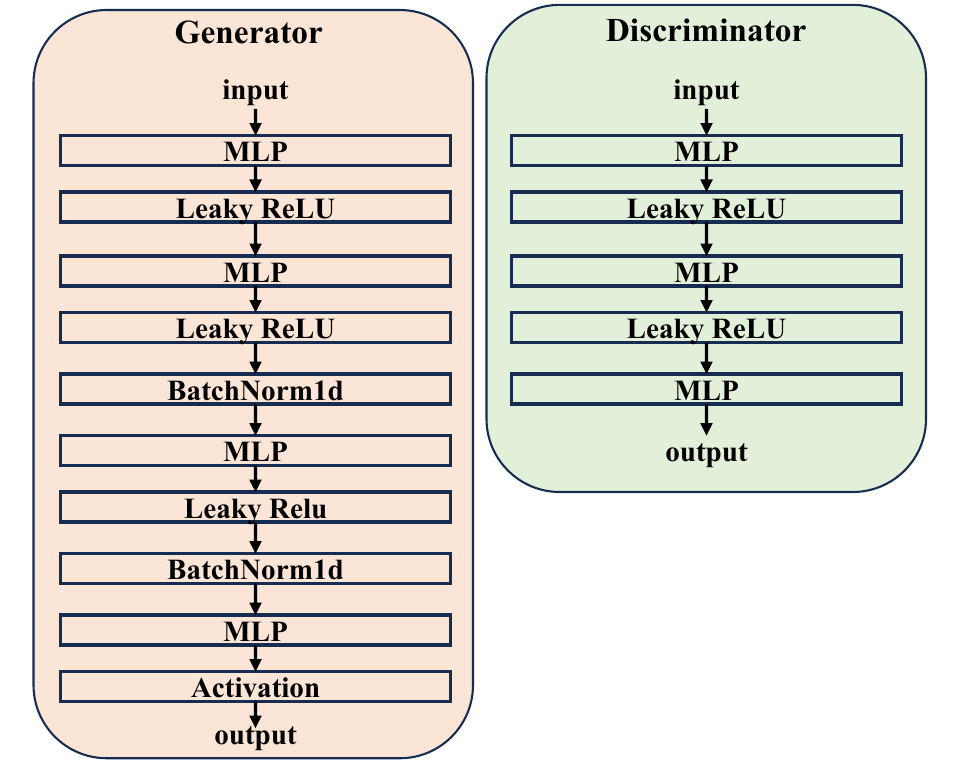}}
\caption{Architectures of WGAN\_GP.}
\label{fig:WGAN_architecture}
\end{figure}

\begin{figure}[htbp]
\center{\includegraphics[width=\linewidth, scale=1.2]{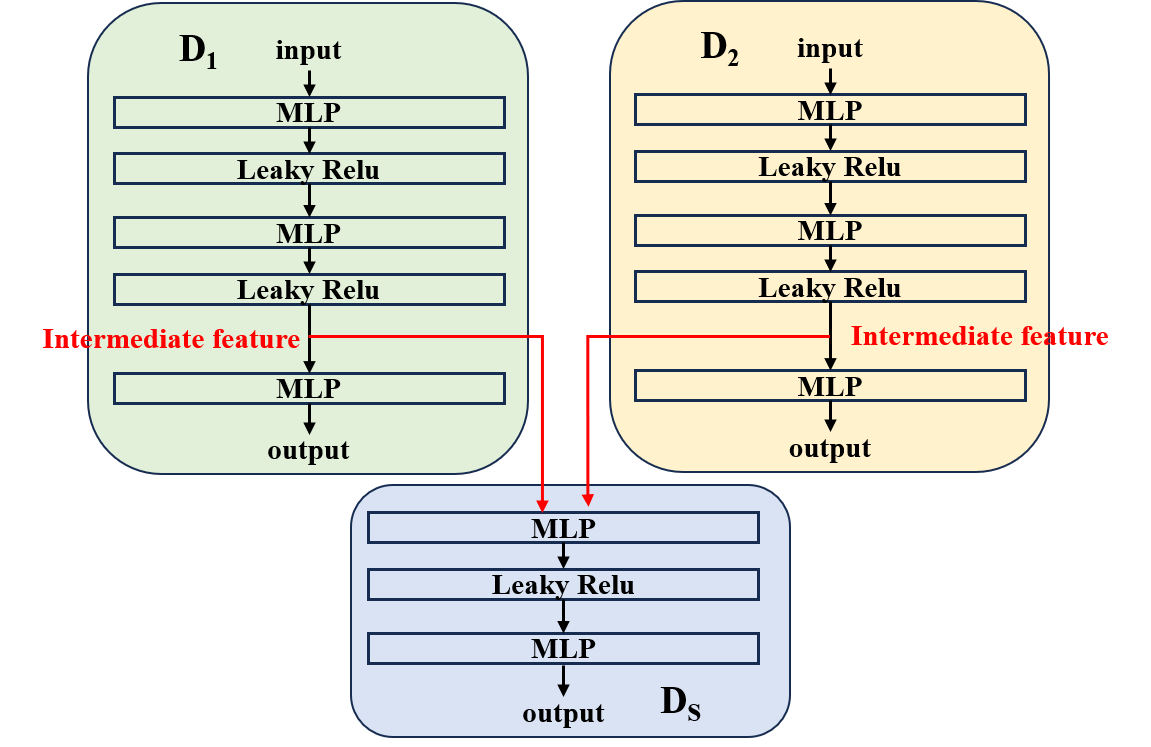}}
\caption{Architectures of the discriminators of VFLGAN.}
\label{fig:VFLGAN_architecture}
\end{figure}


\subsection{Details of VFLGAN-base} \label{app:Details of VFLGAN-base}

Figure \ref{fig:VFLmodel_base} shows the frameworks of VFLGAN-base. To construct VFLGAN-base, we delete the second part (shown with a deeper colour) of $D_1$ and $D_2$ of VFLGAN (shown in Fig. \ref{fig:vflmodel}) and their loss functions ($L_{D_1}$ and $L_{D_2}$). $L_{D_i}$, $i\in\{1,2\}$, can push the distribution of $\tilde{\boldsymbol{x}}_i$ to be similar to that of $\boldsymbol{x}_i$. $L_{D_s}$ can push the distribution of $[\tilde{\boldsymbol{x}}_1,\tilde{\boldsymbol{x}}_2]$ to be similar to that of $[\boldsymbol{x}_1,\boldsymbol{x}_2]$. 
VFLGAN-base serves as an ablation model to evaluate the effectiveness of the second part and loss function of $D_1$ and $D_2$ in VFLGAN.
As shown by our experimental results, the performance of VFLGAN and VFLGAN-base are similar for low-dimensional data. However, for high-dimensional data like images in MNIST datasets, some information will be lost during the concatenation of intermediate features from different parties and the second part of $D_1$ and $D_2$ and their loss functions can help to improve synthetic data quality.
VFLGAN-base is the typical structure of VFL, where one active party ($D_s$) calculates the loss function to update local models in passive parties. In our adjustment of VFLGAN, passive parties (party 1 and party 2) also contribute to their own local model updates.

\begin{figure}[htbp]
\center{\includegraphics[width=\linewidth, scale=1.]{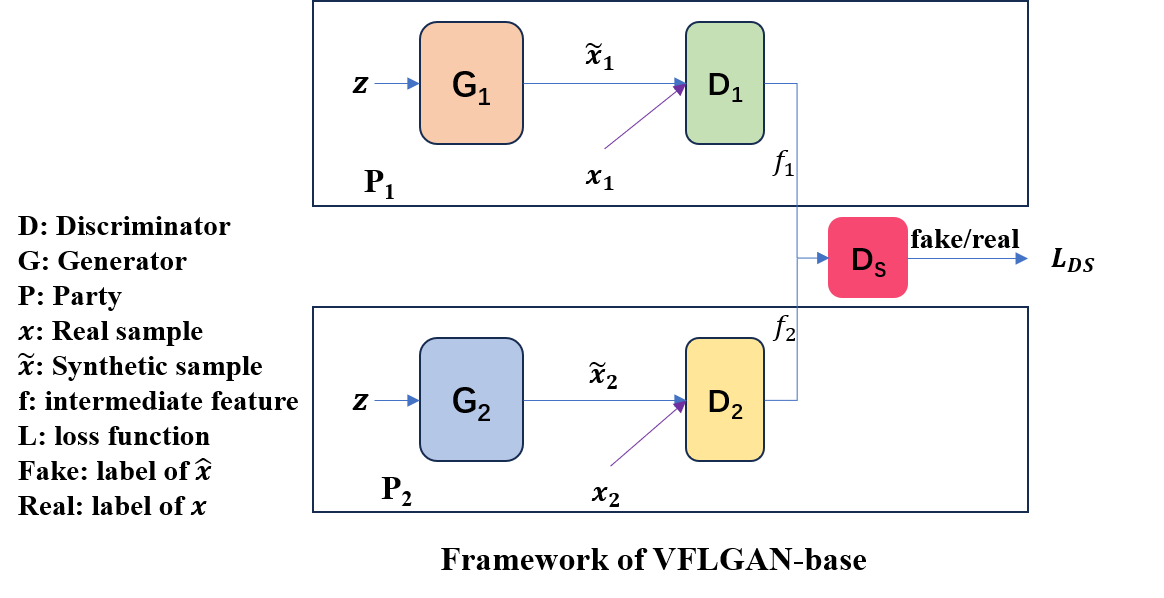}}
\caption{Framework of VFLGAN-base.}
\label{fig:VFLmodel_base}
\end{figure}

\subsection{Supplemental Experiment Results: Normalization for Integer Categorical Attribute} \label{sec:Supplemental Experiment}

In this section, we apply normalization \eqref{eq:normalization} to preprocess the `quality' attribute of red-wine-quality and white-wine-quality datasets. The FD curves during training are shown in Fig. \ref{fig:norm-wine-quality}. We can see that the FD curves of the proposed VFLGAN and VFLGAN-base get very close to the WGAN while the FD curve of VertiGAN is above that of WGAN with a significant gap. The AI training utility is shown in Table \ref{tab:red-wine} and the proposed methods have significantly superior performance compared to VertiGAN. Besides, we can see that the AI utilities in Table \ref{tab:red-wine} are lower than those of corresponding methods in Table \ref{tab:utility classifiction}, which means one-hot representation is more suitable for integer categorical attributes.

\begin{figure}[htbp]
\center{\includegraphics[width=\linewidth, scale=1.]{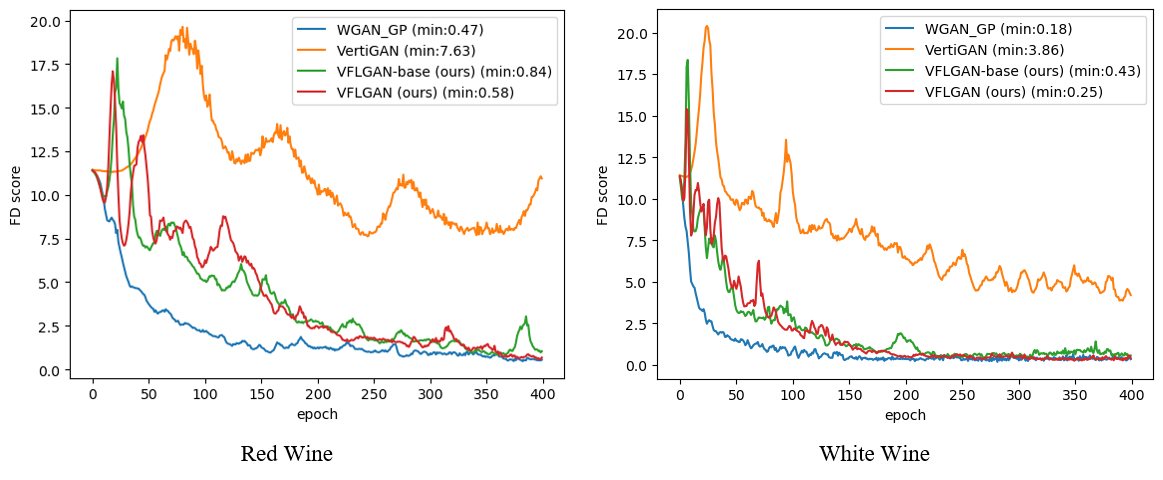}}
\caption{FD curves on the Red-wine-quality dataset (the lower the better).}
\label{fig:norm-wine-quality}
\end{figure}

\subsection{Vulnerable Record Identification} \label{sec:vulnerability}
In \cite{Theresa01}, the authors select records with the most outlier attributes as the most vulnerable records. There may be multiple records that have the most outlier attributes.
Specifically, in our implementation, we select outliers with the following process. First, we select the value of the first quartile and the third quartile of an attribute denoted as $Q_1$ and $Q_4$, respectively. Then, we calculate the threshold by $T=Q_3-Q_1$. Last, we determine the attribute $a$ of a record as an outlier if it satisfies one of the following conditions,
\begin{align}
\text{condition 1:}\quad Q_1 - a > T, \\
\text{condition 2:}\quad    a - Q_3 > T.
\end{align}

In \cite{meeus2023achilles}, the authors calculate the distance between a record and its nearest neighbour record and select the record with the maximum distance as the most vulnerable record. There is usually only one record that has the maximum nearest neighbour distance. Specifically, the distance is calculated by,
\begin{equation}
\begin{aligned}
d\left(x_i, x_j\right):= &1-\frac{\left|\mathcal{F}_{\text {cat }}\right|}{F} \frac{h\left(x_i\right) \cdot h\left(x_j\right)}{\left\|h\left(x_i\right)\right\|_2 *\left\|h\left(x_j\right)\right\|_2}\\
                        &-\frac{\left|\mathcal{F}_{\text {cont }}\right|}{F} \frac{c\left(x_i\right) \cdot c\left(x_j\right)}{\left\|c\left(x_i\right)\right\|_2 *\left\|c\left(x_j\right)\right\|_2},
\end{aligned}
\end{equation}

where $F$ is the number of attributes, $\left|\mathcal{F}_{\text {cat}}\right|$ and $\left|\mathcal{F}_{\text {cont }}\right|$ denote the number of categorical attributes and the number of continuous attributes, respectively, $h(x)$ denotes the one-hot encoding of categorical attributes, and $c(x)$ denotes the vector of continuous attributes. 

In Table \ref{tab:privacy leakage}, we select the most vulnerable record with the method of \cite{meeus2023achilles} as Record 1 and select the most vulnerable record with the method of \cite{Theresa01} as Record 2.
Interestingly, Record 1 of the Adult dataset has both the maximum nearest neighbour distance and the most outlier attributes.

\subsection{Summary of Notation} \label{sec:summary of notation}

All the symbols and notations used in the proposed scheme are defined in Table \ref{tab:Summary of Notation}.

\begin{table}[htbp]
\centering
{
\renewcommand\arraystretch{1.}
    \caption{Symbols and Notations}
    \label{tab:Summary of Notation}
    \begin{tabular}{cc}
    \hline
     \textbf{Notation} & \textbf{Description}\\
    \hline
     \textbf{DP} & differential privacy \\
    \hline
    \textbf{RDP} & Rényi differential privacy \\
    \hline
    \textbf{VFL} & vertical federated learning \\
    \hline
    \textbf{HFL} & horizontal federated learning \\
    \hline
    \textbf{GAN} & generative adversarial network \\
    \hline
    \textbf{ASSD} & Auditing Scheme for Synthetic Datasets \\
    \hline
    \textbf{ASIF} & Auditing Scheme for Intermediate Features \\
    \hline
    \textbf{$P_v$} & distribution of variable $v$ \\
    \hline
    \textbf{$x/X$} & real data / real dataset \\
    \hline
    \textbf{$x_i/X_i$} & real data / real dataset of party $i$ \\
    \hline
    \textbf{$\tilde x/\tilde X$} & synthetic data / synthetic dataset \\
    \hline
    \textbf{$\tilde x_i/\tilde X_i$} & synthetic data / synthetic dataset of party $i$ \\
    \hline
    \textbf{$\lambda$} & balancing coefficient \\
    \hline
    \textbf{$\eta$} & learning rate \\
    \hline
    \textbf{$B$} & size of the mini-batch during training\\    
    \hline
    \textbf{$x^B$} & a mini-batch of real samples \\    
    \hline
    \textbf{$\tilde x^B$} & a mini-batch of synthetic samples \\    
    \hline
    \textbf{$D_i$} & discriminator $i$ (discriminator of party $i$) \\
    \hline
    \textbf{$D^1_i$} & the first part of discriminator $i$ \\
    \hline
    \textbf{$D_s$} & shared discriminator \\
    \hline
    \textbf{$G_i$} & generator $i$ \\
    \hline
    \textbf{$\mathcal{L}_{M}$} & loss function of model $M$ \\
    \hline
    \textbf{$\theta_{M}$} & model $M$'s parameters \\
    \hline
    \textbf{$\theta^1_{M}$} & parameters of model $M$'s first layer \\
    \hline
    \textbf{$\theta^{2:n}_{M}$} & parameters between $M$'s $2^{nd}$ layer to $n^{th}$ layer  \\
    \hline
    \textbf{$\mathcal{G}_{M}$} & gradients of model $M$'s parameters \\
    \hline
    \textbf{$\mathcal{G}^1_{M}$} & gradients of $\theta^1_{M}$ \\
    \hline
    \textbf{$\mathcal{G}^{2:n}_{M}$} & gradients of $\theta^{2:n}_{M}$ \\
    \hline
    \end{tabular}
    }
\end{table}

\end{document}